\documentclass[onecolumn, print]{ieeecolor}
\usepackage{generic}
\usepackage{cite}
\usepackage{amsmath,amssymb,amsfonts}
\usepackage{algorithmic}
\usepackage{graphicx}
\usepackage{algorithm,algorithmic}
\usepackage{hyperref}
\hypersetup{hidelinks=true}
\usepackage{textcomp}
\let\labelindent\relax
\usepackage{enumitem}
\usepackage{pgfplots}
\pgfplotsset{compat=1.18}
\usepgfplotslibrary{fillbetween}

\usepackage{pgfplots}
  \pgfplotsset{compat=newest}
  \usetikzlibrary{plotmarks}
  \usetikzlibrary{arrows.meta}
  \usepgfplotslibrary{patchplots}
  \usepackage{grffile}

\usepackage{microtype}
\usepackage{subcaption}
\usepackage{tikz}
\usetikzlibrary{arrows.meta,positioning,calc}
\usepackage{booktabs}

\usepackage{silence}
\usepackage{amsthm}
\usepackage{mathtools}
\theoremstyle{definition}

\theoremstyle{plain}
\newtheorem{assumption}{Assumption}
\newtheorem{lemma}{Lemma}
\newtheorem{theorem}{Theorem}

\newcommand{\xstar}{x^*}

\newcommand{\RR}{\mathbb{R}}

\newcommand{\afrak}{\mathfrak{a}}
\newcommand{\efrak}{\mathfrak{e}}
\newcommand{\FF}{\mathcal{F}}
\newcommand{\EE}{\mathbb{E}}
\newcommand{\betatilde}{\tilde{\beta}}
\newcommand{\etatilde}{\tilde{\eta}}
\newcommand{\gammatilde}{\tilde{\gamma}}
\newcommand{\boldx}{\mathbf{x}}

\def\BibTeX{{\rm B\kern-.05em{\sc i\kern-.025em b}\kern-.08em
    T\kern-.1667em\lower.7ex\hbox{E}\kern-.125emX}}

\begin{document}
\title{Heavy-Tailed and Long-Range Dependent Noise in Stochastic Approximation: A Finite-Time Analysis}
\author{Siddharth Chandak, Anuj Kumar Yadav, Ayfer Özgür, and Nicholas Bambos
\thanks{\hrule\medskip Siddharth Chandak, Ayfer Özgür, and Nicholas Bambos are with the Department of Electrical Engineering, Stanford University, CA, USA. {Anuj Kumar Yadav is with the School of Computer \& Communication Sciences, EPFL, Lausanne, Switzerland.} Emails: chandaks@stanford.edu, anuj.yadav@epfl.ch, aozgur@stanford.edu, bambos@stanford.edu.}
}

\maketitle

\begin{abstract}
Stochastic approximation (SA) is a fundamental iterative framework with broad applications in reinforcement learning and optimization. Classical analyses typically rely on martingale difference or Markov noise with bounded second moments, but many practical settings, including finance and communications, frequently encounter heavy-tailed and long-range dependent (LRD) noise. In this work, we study SA for finding the root of a strongly monotone operator under these non-classical noise models. We establish the first finite-time moment bounds in both settings, providing explicit convergence rates that quantify the impact of heavy tails and temporal dependence. Our analysis employs a noise-averaging argument that regularizes the impact of noise without modifying the iteration. Finally, we apply our general framework to stochastic gradient descent (SGD) and gradient play, and corroborate our finite-time analysis through numerical experiments.
\end{abstract}


\section{Introduction}
\label{sec:introduction}
Stochastic Approximation (SA) is a class of iterative schemes to find the zeroes of an operator given its noisy realizations~\cite{Robbins-Monro}. A wide range of practical algorithms across different fields can be modeled within the SA framework. These include optimization algorithms such as stochastic gradient descent (SGD) and mirror descent~\cite{bubeck2015convex}, reinforcement learning algorithms such as Q-learning, policy gradient, and TD(0)~\cite{Sutton}, as well as algorithms arising in communications and stochastic control~\cite{Borkar-book}.

Noise plays a central role in stochastic approximation, both in enabling many algorithms to be modeled within this framework and in shaping their behavior. The classical SA formulation assumes martingale difference noise with bounded moments~\cite{Robbins-Monro,Kushner}, which is well suited to statistical estimation and stochastic optimization problems arising from independent sampling or unbiased gradient estimates. More recently, reinforcement learning applications have motivated the study of Markovian noise models, as the learning dynamics are governed by Markov decision processes~\cite{Borkar-Markov, Borkar-book}. Reflecting this emphasis, much of the existing SA theory has been developed under these classical noise models. However, in many real-world applications, these noise models can be overly restrictive.

A first limitation is the assumption of bounded second moments, which excludes noise with heavy tails and occasional large-magnitude fluctuations. To address this, it is natural to consider \textbf{heavy-tailed noise} models, in which the noise only admits finite $p$-th moments for some $p<2$~\cite{heavy-tail-book}. Such models capture settings where rare but extreme events play a significant role, as encountered in queueing systems~\cite{qq}, and finance~\cite{finmarkets}.

A second limitation is the assumption of weak or no temporal dependence, which fails to capture the persistent correlations observed in many time series. This motivates the study of SA under \textbf{long-range dependent (LRD) noise} processes, where correlations decay slowly over time and can substantially affect the dynamics of the iterates~\cite{LRD-book}. Such behavior arises naturally in applications including network traffic~\cite{lrd-ex-1}, climate systems~\cite{lrd-ex-2}, and economic time series~\cite{lrd-ex-3}. 

Motivated by these applications and by the limitations of existing theory, we study stochastic approximation under heavy-tailed and long-range dependent noise models. 

\subsection{Our Contributions}
We establish the first finite-time convergence bounds for SA under both heavy-tailed and LRD noise models. Specifically, we study the problem of finding the root of a strongly monotone operator from noisy observations, and derive bounds on moments of the error $\|x_k - \xstar\|$, where $\xstar$ is the unique zero of the operator. Under martingale difference and Markovian noise, the mean square error at iteration $k$ decays as $\mathcal{O}(k^{-1})$~\cite{Zaiwei}. In contrast, under the noise models considered in this work, the convergence rates degrade as follows:
\begin{itemize}
    \item \textbf{Heavy-tailed noise:} When the noise admits only a finite $p^{\text{th}}$ moment for some $p \in (1,2)$, we establish bounds on the $p^{\text{th}}$-moment of the error, showing a decay rate of $\mathcal{O}(k^{-(p-1)})$. This setting includes, for example, centered Pareto distributions with tail index $\alpha\in(1,2)$ and symmetric $\alpha$-stable noise, with $p=\alpha$ in both cases.
    \item \textbf{LRD noise:} When the autocovariance of the noise process decays as $\mathcal{O}(h^{-\delta})$ for $\delta\in(0,1)$, we obtain a mean square error bound of order $\mathcal{O}(k^{-\delta})$. This setting covers, for instance, fractional Gaussian noise (fGn) with Hurst index $H\in(1/2,1)$ where $\delta=2-2H$ and FARIMA$(0,c,0)$ processes where $\delta = 1-2c$.
\end{itemize}

These results admit a natural interpretation: heavier-tailed noise (smaller $p$) induces larger fluctuations, while stronger temporal dependence (smaller $\delta$) slows averaging, both leading to slower convergence.

We apply our general framework to two important classes of algorithms: stochastic gradient descent (SGD) for strongly convex optimization and gradient play in strongly monotone games. As a consequence of our results in Theorems~\ref{thm:heavy} and~\ref{thm:LRD}, we obtain the first finite-time guarantees for SGD under LRD noise, and for gradient play under both heavy-tailed and LRD noise. We also provide numerical experiments illustrating the impact of heavy tails and temporal dependence on convergence.

Our proof technique relies on introducing averaged noise and auxiliary iterates, transforming the iteration so that randomness appears only through an averaged term. This averaging regularizes the noise, yielding improved moment bounds even under heavy-tailed or long-range dependent (LRD) perturbations. We emphasize that analyzing this averaged noise sequence is just a proof technique and not a modification to the iteration.

\subsection{Related Work}
Classical analyses of SA primarily focused on asymptotic convergence under martingale difference or Markovian noise (see~\cite{Borkar-book, Kushner} for textbook references). Motivated by applications in optimization and reinforcement learning, recent work has focused on finite-time guarantees. However, these non-asymptotic results are largely developed under the same classical noise assumptions (e.g., see~\cite{Zaiwei, srikant2019finite} and references therein).

To the best of our knowledge, the only existing work on general SA under heavy-tailed or LRD noise focuses on asymptotic convergence guarantees~\cite{anantharam2012stochastic}. In particular, they show almost sure convergence for LRD noise, which we complement by providing mean square error bounds, and convergence in the $p^{\text{th}}$-moment for heavy-tailed noise, which we extend by establishing explicit finite-time bounds.

Despite this limited literature for general stochastic approximation, stochastic gradient descent (SGD) under heavy-tailed noise has been widely studied. There is a substantial body of work that analyzes \textit{vanilla} SGD under heavy-tailed gradient noise~\cite{vanilla-sgd-1, vanilla-sgd-2, vanilla-sgd-3, vanilla-sgd-4, vanilla-sgd-5}. While our framework recovers the same rates in the strongly convex setting, these works exploit properties specific to gradients and therefore do not extend to the general strongly monotone operators considered in this work. Other works propose modifications to the algorithm such as norm-based clipping~\cite{clipping-1, clipping-2}, and gradient normalization~\cite{normalization}.

\subsection{Outline and Notation}
This paper is structured as follows. Section~\ref{sec:formulation} introduces the general SA framework. Section~\ref{sec:results} presents the heavy-tailed and long-range dependent noise models and develops the corresponding bounds. Section~\ref{sec:outline} provides a proof sketch of the main results. Section~\ref{sec:applications} presents SGD and gradient play as applications, together with numerical simulations. Section~\ref{sec:conc} concludes with a discussion of future directions.

Throughout this work, $\|\cdot\|$ denotes the Euclidean norm, and $\langle x_1, x_2 \rangle$ denotes the inner product $x_1^\top x_2$. We use the notation $f(k)=\mathcal{O}(g(k))$ to denote that there exists a constant $C>0$ such that $|f(k)|\le Cg(k)$ for all $k\geq 0$.

\section{Problem Formulation}\label{sec:formulation}
In this section, we formulate the stochastic approximation (SA) problem, and present the assumptions that are common across different noise models considered in this work. Consider the following iteration,
\begin{equation}\label{eqn:iter}
    x_{k+1}=x_k-\beta_k(F(x_k)+\eta_k).
\end{equation}
Here, $x_k\in\RR^d$ is the iterate and $\beta_k$ is the stepsize at time $k$. The function $F:\RR^d\mapsto\RR^d$ denotes the mapping we wish to find the zero for, and $\eta_k$ denotes the noise sequence. We consider stepsize sequence of the following form:
$$\beta_k=\frac{\beta}{k+K_0},$$
where $\beta,K_0>0$. This decaying stepsize sequence is standard in analysis of stochastic approximation, and allows for $\mathcal{O}(1/k)$ mean square error bound under the light-tailed martingale difference noise model~\cite{Zaiwei}.

Our first assumption imposes strong monotonicity and Lipschitz continuity on the operator $F(\cdot)$.
\begin{assumption}\label{assu:monotone}
    The operator $F:\RR^d\mapsto\RR^d$ is $\mu$-strongly monotone, i.e., there exists a $\mu>0$ such that
    $$\langle F(x_1)-F(x_2),x_1-x_2\rangle\geq \mu\|x_1-x_2\|^2,$$
    for all $x_1,x_2\in\RR^d$. Moreover, $F(\cdot)$ is $L$-Lipschitz, i.e., 
    $$\|F(x_1)-F(x_2)\|\leq L\|x_1-x_2\|,$$
    for all $x_1,x_2\in\RR^d$ where $L>0$.
\end{assumption}
This assumption is common in finite-time analyses of stochastic approximation, as it yields strong convergence guarantees while encompassing a broad class of problems. Examples include gradient operators for strongly convex objectives, \textit{pseudogradient} operators in strongly monotone games, and linear operators induced by Hurwitz matrices. In these cases, \eqref{eqn:iter} specializes to stochastic gradient descent (SGD), gradient play, and linear SA, respectively. Such operators have a unique zero~\cite[Theorem 2.3.3 (b)]{VI}, which we denote by $\xstar$, i.e., there exists a unique $\xstar\in\RR^d$ such that $F(\xstar)=0$. Our goal here is to study the convergence rate of $x_k$ to $\xstar$. 


\section{Noise Models and Results}\label{sec:results}
We now present the different noise models considered in this paper, along with the corresponding finite-time guarantees. 

\subsection{Martingale Difference with Bounded Second Moment}
Although this noise model is not the main focus of this paper, we include the standard martingale difference noise with bounded second moments for completeness. This provides an useful benchmark to contrast the resulting guarantees and to motivate the need for different proof techniques in the presence of heavy-tailed or long-range dependent (LRD) noise. Suppose that $\eta_k$ satisfies the following assumption.
\begin{assumption}[\textbf{Martingale difference with bounded second moment}]\label{assu:standard}
    Let $\FF_k$ denote a sigma-field defined as $\FF_k:=\sigma(x_0,\eta_0,\eta_1,\dots,\eta_{k-1})$. Then, $\{\eta_k\}_{k \geq 0}$ is a martingale difference sequence adapted to the filtration $\{\FF_k\}_{k\geq0}$, i.e., $\EE[\eta_k\mid\FF_k]=0$. Moreover, for all $k\geq 0$, $\EE[\|\eta_k\|^2\mid\FF_k]\leq \sigma^2$. 
\end{assumption}
This assumption often arises naturally when we observe noisy observations of the operator $F(\cdot)$, e.g., $F(x_k,\xi_k)$ at time $k$ such that $\EE[F(x_k,\xi_k)\mid\FF_k]=F(x_k)$, where $\xi_k$ is some random variable that is not $\FF_k$-measurable. Then, defining $\eta_k=F(x_k,\xi_k)-F(x_k)$ satisfies the martingale difference assumption. Moreover, if $F(\cdot)$ is bounded, then the $\eta_k$ is also bounded and therefore has finite variance. Sub-Gaussian and sub-exponential distributions are two common examples of distribution families with bounded second moments. 

We now present the following result which shows that under martingale difference noise with bounded second moment, the mean square error bound is $\mathcal{O}(1/k)$. 
\begin{theorem}\label{thm:standard}
    Suppose Assumptions~\ref{assu:monotone} and~\ref{assu:standard} are satisfied. Then there exist constants $C_1, C_2$, and $C_3$ such that if $\beta>C_1$, $K_0\geq C_2$, then for all $k\geq 0$,
    $$\EE\left[\|x_k-\xstar\|^2\right]\leq \frac{C_3}{k+K_0}.$$
\end{theorem}
Explicit values for $C_1,C_2,$ and $C_3$ are provided along with the theorem's proof in Appendix~\ref{app:thm_proof_standard}. 
\subsection{Heavy-tailed Noise}
We now turn to the first major noise model considered in this work: heavy-tailed noise. Our formal assumption focuses on noise sequences with unbounded second moments, which capture the presence of rare but large-magnitude fluctuations commonly observed in real-world systems, such as financial markets~\cite{finmarkets} and queueing systems~\cite{qq}.

\begin{assumption}[\textbf{Heavy-tailed noise}]\label{assu:heavy}
   The noise sequence $\{\eta_k\}_{k \geq 0}$, where $\eta_k \in\RR^d$, is an independent, zero-mean sequence with unbounded second moment but bounded $p^{\text{th}}$ moment for some $1<p<2$. That is, there exists $\sigma>0$ such that 
$$\EE[\eta_k]=0,\;\;\EE[\|\eta_k\|^2]=\infty,\;\;\text{and}\;\;\EE[\|\eta_k\|^p]\leq \sigma^p,$$
  for all $k\geq 0$.
\end{assumption}
We note that there exist heavy-tailed distributions with finite second moment. However, such cases can be handled under Assumption~\ref{assu:standard}. In line with the optimization literature~\cite{vanilla-sgd-3, clipping-1}, we adopt unbounded second moment as the defining characteristic of heavy-tailed noise in this paper. 

Heavy-tailed distributions with finite $p^\text{th}$ moments but infinite second moments arise naturally in many applications, and standard examples include Pareto laws, $\alpha$-stable laws, and related power-law–type models. For instance, a centered Pareto distribution with tail index $\alpha\in(1,2)$ has a tail which decays polynomially~\cite{Foss-Pareto}. Similarly, $\alpha$-stable distributions with stability index $\alpha\in(1,2)$ generalize Gaussian noise~\cite{Samo-alpha-stable}, with the Gaussian case recovered when $\alpha=2$. Both of these distributions have finite $p^\text{th}$ moments for all $1< p<\alpha<2$. Further examples discussed in the heavy-tail literature include certain lognormal, Weibull with shape less than $1$, and Student-$t$ distributions, where the tail index again directly controls which moments exist~\cite{Foss-Pareto}.

We now present our first main result, which provides a bound on the $p^\text{th}$ moment of the error.
\begin{theorem}\label{thm:heavy}
    Suppose Assumptions~\ref{assu:monotone} and~\ref{assu:heavy} are satisfied. Then there exist constants $C_4, C_5$, and $C_6$ such that if $\beta>C_4$, $K_0\geq C_5$, then for all $k\geq 0$ we have,
    $$\EE\left[\|x_k-\xstar\|^p\right]\leq \frac{C_6}{(k+K_1)^{p-1}}.$$
\end{theorem}
Explicit values for $C_4,C_5,$ and $C_6$ have been provided along with the theorem's proof in Appendix~\ref{app:thm_proof_heavy}. An outline of the proof is given in Section~\ref{sec:outline} through a series of lemmas. The result shows that when the noise is heavy-tailed with only a finite $p^\text{th}$ moment, the $p^\text{th}$ moment of the error decays at rate $\mathcal{O}(1/k^{p-1})$. By Jensen's inequality, we obtain that $$\EE[\|x_k-\xstar\|]=\mathcal{O}\left(\frac{1}{k^{(p-1)/p}}\right).$$ Thus, weaker moment assumptions lead to slower convergence, which aligns with the intuition. The constant $C_6$ scales proportionally with $\sigma^p$, capturing how the convergence rate depends on the noise magnitude.

\subsection{Long-Range Dependent Noise}
We now introduce a model with temporally correlated noise. To capture such correlations, we consider LRD noise sequences, which exhibit persistent temporal dependence patterns commonly observed in real-world time series such as network traffic~\cite{lrd-ex-1}, climate data~\cite{lrd-ex-2}, and financial markets~\cite{lrd-ex-3}.
\begin{assumption}[\textbf{Long-range dependent noise}] \label{assu:LRD}
The noise sequence $\{\eta_k\}_{k \geq 0}$, where $\eta_k \in\RR^d$, is a zero-mean, weakly stationary process with autocovariance $\gamma(h)=\EE[\langle \eta_0,\eta_h\rangle]$. The autocovariance sequence is not absolutely summable, i.e., 
$$\sum_{h=0}^\infty |\gamma(h)|=\infty.$$
In addition, there exist constants $\sigma>0$ and $\delta\in(0,1)$ such that for all $h\geq 0$,
$$|\gamma(h)|\leq \sigma^2(1+h)^{-\delta}.$$
\end{assumption}
We note that the definition of LRD only requires the non-summability of the autocovariance sequence~\cite{LRD-book}. We impose the additional polynomial decay to facilitate finite-time analysis of the SA iteration. For vector-valued noise sequences, the autocovariance is typically defined via the outer product, yielding a matrix-valued function. Here, we impose assumptions only on the inner-product covariance, as this is sufficient for our analysis.


A widely studied example of an LRD sequence is fractional Gaussian noise (fGn), defined as the incremental process of fractional Brownian motion (fBm). Fractional Brownian motion $\{B_H(t)\}_{t\ge 0}$ is a zero-mean Gaussian process parameterized by the Hurst index $H\in(0,1)$, with covariance
$$\EE[B_H(t)B_H(s)]
= \tfrac12\big(t^{2H}+s^{2H}-|t-s|^{2H}\big).$$
When $H>\tfrac12$, the increments $\{B_H(t+1)-B_H(t)\}_{t \geq 0}$ form a stationary Gaussian process whose autocovariance decays as $\gamma(h)\asymp h^{2H-2}$, and hence exhibits long-range dependence.
Another prominent class of LRD models is given by fractionally integrated autoregressive moving average (FARIMA) processes. These extend classical ARIMA models by replacing the usual integer differencing operator $(1-L)^m$, where $m$ is an integer, with a fractional operator $(1-L)^c$, where $c\in(0,\tfrac12)$ and $L$ denotes the lag operator. This fractional differencing yields an autocovariance that decays as 
$\gamma(h) \asymp h^{2c-1}$~\cite{LRD-book}.

We now present the mean square error bound corresponding to the LRD noise model.
\begin{theorem}\label{thm:LRD}
    Suppose Assumptions~\ref{assu:monotone} and~\ref{assu:LRD} are satisfied. Then there exist constants $C_7, C_8$, and $C_9$ such that if $\beta>C_7$, $K_0\geq C_8$, then for all $k\geq 0$, we have
    $$\EE\left[\|x_k-\xstar\|^2\right]\leq \frac{C_9}{(k+K_0)^\delta}.$$
\end{theorem}
Explicit values for $C_7,C_8,$ and $C_9$ have been provided along with the theorem's proof in Appendix~\ref{app:thm_proof_LRD}. An outline for the proof has been given in Section~\ref{sec:outline} through a series of lemmas. The result shows that the mean square error decays at the same rate as the autocovariance sequence of the noise. In particular, if the autocovariance satisfies $|\gamma(h)|= \mathcal{O}(h^{-\delta})$, then the bound on $\EE[\|x_k-\xstar\|^2]$ scales as $\mathcal{O}(k^{-\delta})$. This is intuitive: when correlations decay slowly, past noise terms continue to influence the iterates for a longer duration, which slows convergence. Conversely, faster decay of autocorrelation weakens this persistence effect and leads to faster error reduction. Similar to the heavy-tailed noise setting, the constant $C_9$ scales proportionally with $\sigma^2$, capturing the dependence of the noise magnitude on the convergence. 

\section{Proof Outlines}\label{sec:outline}
In this section, we outline the proofs of Theorem~\ref{thm:heavy} and Theorem~\ref{thm:LRD} via a sequence of lemmas, whose proofs are deferred to Appendix~\ref{app:lemma_proof}. The proof techniques that allow us to handle heavy-tailed and long-range dependent (LRD) noise rely on an equivalent fixed-point formulation and on the introduction of averaged noise sequences together with suitable auxiliary iterates. Before presenting these ideas, however, we first explain why the traditional finite-time analysis for martingale difference noise with bounded second moment does not extend to other noise models considered in this work.

\subsection{Why Does Traditional Analysis Fail?}
The proof of Theorem~\ref{thm:standard} for light-tailed martingale difference noise is standard in the stochastic approximation literature. It is based on a recursion for $\EE[\|x_k-\xstar\|^2]$ obtained by directly analyzing the squared error $\|x_k-\xstar\|^2$. We summarize the key steps below and defer the full proof to Appendix~\ref{app:thm_proof_standard}. We begin with the following observation.
\begin{align*}
    \|x_{k+1}-\xstar\|^2&=\|x_k-\xstar-\beta_kF(x_k)\|^2\\
    &\;\;+2\langle x_k-\xstar-\beta_kF(x_k),\beta_k\eta_k\rangle+\beta_k^2\|\eta_k\|^2.
\end{align*}
Under suitable assumptions on the stepsize $(\beta_k \leq \mu/L^2)$, and after simplifying using the strong monotonicity of the mapping $F(\cdot)$, taking conditional expectation yields
\begin{subequations}\label{split-trad}
    \begin{align}
    \EE&\left[\|x_{k+1}-\xstar\|^2\mid\FF_k\right]\nonumber\\
    &\leq (1-\mu\beta_k)\EE\left[\|x_k-\xstar\|^2\mid\FF_k\right]\label{split-trad-1}\\
    &\;\;+2\EE\left[\langle x_k-\xstar-\beta_kF(x_k),\beta_k\eta_k\rangle\mid\FF_k\right]\label{split-trad-2}\\
    &\;\;+\beta_k^2\EE\left[\|\eta_k\|^2\mid\FF_k\right]\label{split-trad-3}.
\end{align}
\end{subequations}
Then under the assumption that $\eta_k$ is a light-tailed martingale difference sequence, the term \eqref{split-trad-2} is zero and the term \eqref{split-trad-3} can be bounded by $\beta_k^2\sigma^2$. Hence, \eqref{split-trad} can then be simplified to obtain the following recursion.
\begin{align*}
    \EE\left[\|x_{k+1}-\xstar\|^2\right]\leq (1-\mu\beta_k)\EE\left[\|x_k-\xstar\|^2\right]+\beta_k^2\sigma^2.
\end{align*}
This recursion can then be solved to obtain a bound of $\mathcal{O}(1/k)$. 

This analysis can be used only for light-tailed martingale difference sequences. For heavy-tailed noise, the term \eqref{split-trad-3} is unbounded (as the second moment of the noise is unbounded). And for LRD noise, the term \eqref{split-trad-2} is not zero due to temporal correlation in the noise. Therefore a different analysis is required for both heavy-tailed and LRD noise models.

\subsection{Proof Technique}
Our proof relies on modifying the iteration \eqref{eqn:iter} into a form in which the noise can be `partially separated' from the iterates and subsequently averaged. This is done in two steps.

Although we formulate our problem as finding the zero (root) of a strongly monotone operator, it can equivalently be reformulated as finding the fixed point of a mapping $G(\cdot)$ that is contractive under the Euclidean norm. While this equivalence is well known in the optimization literature, we emphasize it here because it allows us to cleanly isolate the noise term. The following lemma formalizes this equivalence and provides the corresponding reformulation.
\begin{lemma}\label{lemma:contrac}
Suppose the operator $F(\cdot)$ is $\mu$-strongly monotone and $L$-Lipschitz (Assumption~\ref{assu:monotone}). Then,
    \begin{enumerate}[label=\alph*)]
        \item For $\zeta=\mu/L^2$, the map $G(x)=x-\zeta F(x)$ is $\lambda$-contractive under the Euclidean norm, i.e., 
        $$\|G(x_1)-G(x_2)\|\leq \lambda\|x_1-x_2\|,\quad \forall x_1,x_2\in\RR^d,$$
        where the contraction factor is $\lambda=\sqrt{1-\mu^2/L^2}$. Moreover, $\xstar$ is the unique fixed point of the map $G(\cdot)$ i.e., $G(\xstar)=\xstar$.
        \item The iteration in \eqref{eqn:iter} can be rewritten as 
        $$x_{k+1}=x_k+\betatilde_k(G(x_k)-x_k+\etatilde_k),$$
        where $\betatilde_k=\frac{\betatilde}{(k+K_0)}$ with $\betatilde=\beta/\zeta$ and $\etatilde_k=-\zeta\eta_k$.
    \end{enumerate}
\end{lemma}

We now define the averaged noise sequence $U_{k+1}=(1-\betatilde_k)U_k+\betatilde_k\etatilde_k$ with $U_0=0$. Note that this averaging is introduced purely as a proof technique and not as a modification to the algorithm. On expanding the recursion, we get 
$$U_k=\sum_{i=0}^{k-1}\betatilde_i\prod_{j=i+1}^{k-1}(1-\betatilde_j)\etatilde_i.$$ We also define the modified iterate $z_k=x_k-U_k$ for all $k\geq 0$. The following lemma expresses the original error in terms of $z_k$, and provides a reformulation of the iteration.
\begin{lemma}\label{lemma:noise_avg}
Suppose Assumption~\ref{assu:monotone} is satisfied. Then,
    \begin{enumerate}[label=\alph*)]
        \item For all $k\geq 0$, and exponent $1\leq q\leq 2$, we have
        $$\EE\left[\|x_k-\xstar\|^q\right]\leq 2\EE\left[\|z_k-\xstar\|^q\right]+2\EE\left[\|U_k\|^q\right]$$
        \item The iteration \eqref{eqn:iter} can be rewritten as:
        \begin{equation}\label{iter-alter}
            z_{k+1}=z_k+\betatilde_k(G(z_k)-z_k+\Delta_k),
        \end{equation}
        where $\Delta_k=G(x_k)-G(z_k)$, and $\|\Delta_k\|\leq \|U_k\|.$
    \end{enumerate}
\end{lemma}
Lemma~\ref{lemma:noise_avg} shows that it suffices to control the averaged noise sequence $U_k$ and to analyze the modified recursion \eqref{iter-alter}. Part (a) reduces the study of $\EE[\|x_k-\xstar\|^q]$ to bounding the corresponding quantity for the auxiliary iterates $z_k$ together with the $q^\text{th}$ moment of $U_k$. Part (b) shows that, after the change of variables, the perturbation enters the recursion only through $\Delta_k$, whose magnitude is directly controlled by $\|U_k\|$.

The key idea behind introducing $U_k$ and the auxiliary sequence ${z_k}$ is therefore to rewrite the original SA iteration in a form where the randomness appears only through an averaged noise term. Averaging regularizes the noise sequence: even when the noise is heavy-tailed or exhibits long-range dependence (LRD), its averaged version admits significantly improved moment bounds and cleaner analysis. Similar techniques have been used in~\cite{Chandak-TTS-opti} and~\cite{Bravo} for analysis of two-time-scale SA and non-expansive SA, respectively.

We now study the heavy-tailed and LRD noise models separately.

\subsubsection{Heavy-Tailed Noise} 
For the heavy-tailed noise sequence $\{\eta_k\}_{k \geq 0}$ with bounded $p^{\text{th}}$ moment, bounds on $\EE[\|U_k\|^q]$ can only be established for $q \le p$. Consequently, only moments of $\|x_k-\xstar\|$ up to order $p$ can be controlled under the heavy-tailed noise model.
\begin{lemma}\label{lemma:heavy}
    Suppose Assumptions~\ref{assu:monotone} and~\ref{assu:heavy} are satisfied. Then,
    \begin{enumerate}[label=\alph*)]
        \item The $p^\text{th}$ moment of the averaged noise $U_k$ decays as follows:
        $$\EE[\|U_k\|^p]\leq 4\zeta^p\sigma^p\betatilde_k^{p-1}=4\zeta\sigma^p\left(\frac{\beta}{k+K_0}\right)^{p-1}.$$
        \item For all $k\geq 0$, we have
        \begin{align*}
            &\EE\left[\|z_k-\xstar\|^p\right]\\
            &\leq \|x_0-\xstar\|^p\left(\frac{K_0}{k+K_0}\right)+\frac{144\zeta\sigma^p}{(1-\lambda)^2}\left(\frac{\beta}{k+K_0}\right)^{p-1}.
        \end{align*}
    \end{enumerate}
\end{lemma}

\subsubsection{Long-Range Dependent Noise}
For the long-range dependent noise sequence $\{\eta_k\}_{k \geq 0}$ with parameter $\delta$, i.e., the autocovariance function decays at the rate $\mathcal{O}\left(h^{-\delta}\right)$, the averaged noise and the error in SA iteration both decay at the rate $\mathcal{O}(1/k^{\delta})$.
\begin{lemma}\label{lemma:LRD}
    Suppose Assumptions~\ref{assu:monotone} and~\ref{assu:LRD} are satisfied. Then,
    \begin{enumerate}[label=\alph*)]
        \item The second moment of the averaged noise $U_k$ decays as follows:
        $$\EE[\|U_k\|^2]\leq \left(\frac{6\zeta^2\sigma^2}{1-\delta}\right)k^{1-\delta}\betatilde_k\leq\left(\frac{6\zeta\sigma^2}{1-\delta}\right)\frac{\beta}{(k+K_0)^{\delta}}.$$
        \item For all $k\geq 0$, we have
        \begin{align*}
            &\EE\left[\|z_k-\xstar\|^2\right]\\
            &\leq \|x_0-\xstar\|^2\frac{K_0}{k+K_0}+\frac{72\zeta\sigma^2}{(1-\lambda)^2(1-\delta)}\frac{\beta}{(k+K_0)^\delta}.
        \end{align*}
    \end{enumerate}
\end{lemma}

\section{Applications}\label{sec:applications}
In this section, we present two algorithms that fall within the SA framework and for which heavy-tailed and long-range dependent noise arise naturally. We also provide numerical simulations that corroborate our theoretical guarantees. In these simulations, we plot the $\ell_2$ error, $\|x_k-\xstar\|$, against the iteration index $k$. Except for single-run plots, results are averaged over $1000$ independent runs. To illustrate the variability across runs, we also report the $10\%$--$90\%$ quantile band.

\subsection{SGD for Strongly Convex Optimization}
Stochastic Gradient Descent (SGD) is a fundamental algorithm in stochastic approximation, with widespread applications in machine learning, statistical estimation, signal processing, and large-scale optimization~\cite{duchi-sgd,lms,stat}. It generates a sequence of iterates according to
\[
x_{k+1} = x_k - \beta_k \bigl( \nabla g(x_k) + \eta_k \bigr),
\]
where $\beta_k$ is the step size, $g : \mathbb{R}^d \to \mathbb{R}$ is the objective function to be minimized, and $\eta_k$ denotes the stochastic error (noise sequence) in the gradient evaluation. Equivalently, the update uses noisy gradient samples $\widehat{\nabla g}(x_k) = \nabla g(x_k) + \eta_k$. For strongly convex and smooth objective functions, the operator $\nabla g(\cdot)$ is strongly monotone and Lipschitz. Under suitable assumptions on the stepsize sequence and the noise process, the iterates converge to the minimizer of $g(\cdot)$.

While classical analyses of SGD often assume i.i.d.\ finite-variance noise, stochastic perturbations in practice can be considerably more complex. In modern machine learning, both empirical and theoretical studies suggest that gradient noise may exhibit heavy-tailed behavior, which motivates the use of stable or other heavy-tailed perturbation models~\cite{ht-sgd-1,ht-sgd-2}. Moreover, when updates are generated from temporally correlated data streams, delayed feedback, momentum-like effects, or colored environmental disturbances, the noise may also exhibit temporal dependence, motivating long-range dependent (LRD) noise models~\cite{lrd-sgd-2}.

As the learning scheme fits our general SA framework, our finite-time bounds (Theorems~\ref{thm:heavy} \&~\ref{thm:LRD}) apply under the corresponding noise models.\\

\begin{figure*}[t]
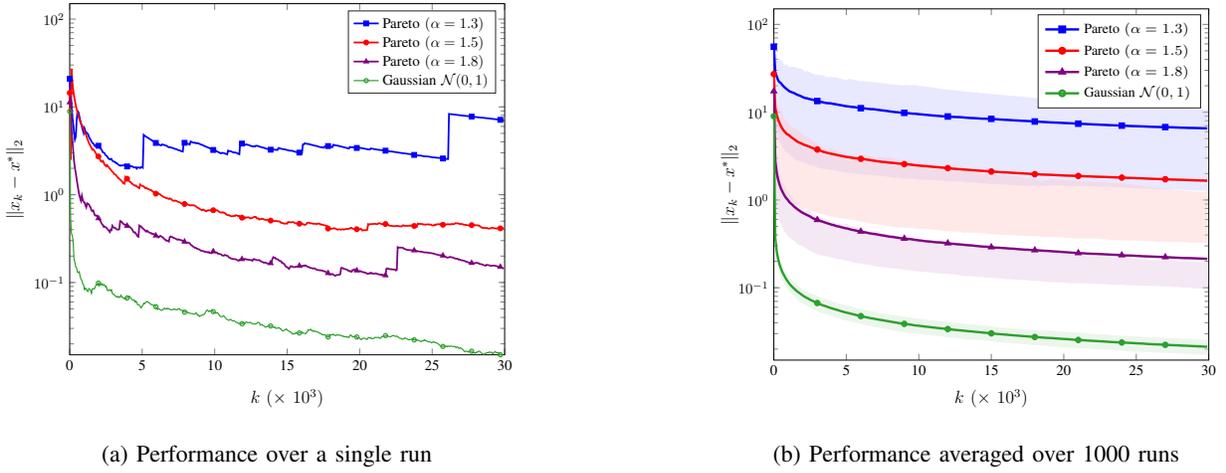

\centering

\begin{subfigure}[t]{0.48\textwidth}
  \centering
  \scalebox{0.75}{\input{figures/ht_sgd_single}}  \vspace{2mm}
  \caption{Performance over a single run}
  \label{fig:single-ht-sgd}
\end{subfigure}\hfill
\hspace*{-3mm}
\begin{subfigure}[t]{0.48\textwidth}
  \centering
  \scalebox{0.75}{\input{figures/ht_sgd_multiple}}
    \vspace{2mm}
  \caption{Performance averaged over 1000 runs}
  \label{fig:multiple-ht-sgd}
\end{subfigure}
\caption{Performance of SGD under heavy-tailed noise ($\alpha-$Pareto distribution)}
\label{fig:sgd-heavy}
\end{figure*}

\begin{figure*}[t]
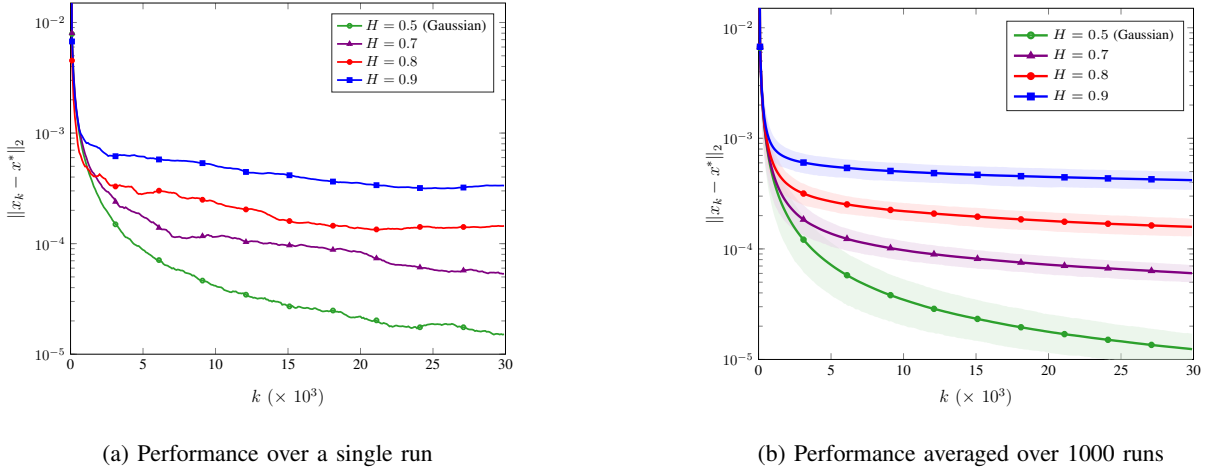

\centering
\begin{subfigure}[t]{0.49\textwidth}
  \centering
  \scalebox{0.75}{\input{figures/lrd_sgd_single}}
  \vspace{2mm}
  \caption{Performance over a single run}
  \label{fig:single-lrd-sgd}
\end{subfigure}\hfill
\hspace*{2mm}
\begin{subfigure}[t]{0.48\textwidth}
  \centering
  \scalebox{0.75}{\input{figures/lrd_sgd_multiple}}
    \vspace{2mm}
  \caption{Performance averaged over 1000 runs}
  \label{fig:multiple-lrd-sgd}
\end{subfigure}
    \vspace{3mm}
\caption{Performance of SGD under LRD noise (fractional Gaussian noise (fGn))}
\label{fig:sgd-LRD}
\end{figure*}

\noindent\textbf{Numerical Simulations}.
We consider a strongly convex function $g(\cdot)$ for our experiments as described below,
\begin{align}
    g(x)=\frac{1}{2}\|Ax-b\|^2 + \sum_{i=1}^{d}\phi_\delta(x_i),
\end{align}
where $\phi_\delta(\cdot)$ is the Huber loss function with threshold $\delta =1$, commonly used in robust regression~\cite{huber1992robust}. The matrix $A \in \mathbb{R}^{m \times d}$ and the vector $b \in \mathbb{R}^{d}$ (with $m =60$ and $d =30$) are sampled as a random Gaussian matrix and vector, respectively. Although it is not constructed explicitly to enforce strong convexity, since $m>d$, such a matrix is almost surely full column rank, so $A^\top A$ is positive definite and the objective function $g$ is strongly convex. The stepsize sequence is chosen as $\beta_k=1/(k+1)$.

For the heavy-tailed experiments, we consider centered Pareto noise with shape parameter $\alpha \in (1,2)$ and scale parameter $1$. In Figures~\ref{fig:single-ht-sgd} and~\ref{fig:multiple-ht-sgd}, we plot the $\ell_2$ error for a single run and the error averaged over $1000$ independent runs, respectively. The results show that smaller values of $\alpha$, corresponding to heavier-tailed noise, lead to slower convergence toward the minimizer. As discussed earlier, heavy-tailed noise is characterized by rare but large-magnitude fluctuations. In the single-run plot, these appear as ``spikes'', while in the averaged results, the mean tends to lie closer to the upper quantiles of the empirical distribution (e.g., the $90\%$ quantile) rather than near the median.

For the LRD experiments, we consider fractional Gaussian noise (fGn) with zero mean, unit variance, and Hurst parameter $H \in (1/2,1)$. The noise is temporally correlated, with $H=0.5$ corresponding to standard Gaussian white noise, and $H>0.5$ corresponding to persistent long-range dependent noise. In the simulations, the generated fGn is scaled by a factor of $20$. In Figures~\ref{fig:single-lrd-sgd} and~\ref{fig:multiple-lrd-sgd}, we plot the $\ell_2$ error for a single run and the error averaged over $1000$ independent runs, respectively. The results show that larger values of $H$, corresponding to stronger temporal dependence in the noise, lead to slower convergence toward the minimizer.


\subsection{Gradient Play in Strongly Monotone Games}
Consider a continuous-action game with $N$ players. Each player $n \in [N]$ takes action $x_k^{(n)}\in\RR^D$ at time $k$. We use $\boldx_k=(x_k^{(1)}, \ldots, x_k^{(N)})$ to denote the $ND$-dimensional concatenation of all players' actions at time $k$. Each player $n$ has utility $u_n(\boldx_k)$ at time $k$ which is a function of all players' actions and they wish to converge to the Nash equilibrium (NE). A NE is an action profile at which no player benefits, i.e., improve their utility from deviating unilaterally. We consider the class of strongly monotone games where there exists a unique pure Nash equilibrium, and the players can converge to this unique NE by performing gradient ascent on their utilities. To formally define such games, we first define the gradient operator as follows $$H(\boldx)\coloneqq \left(\nabla_{x^{(1)}}u_1(\boldx),\ldots,\nabla_{x^{(N)}u_N(\boldx)}\right).$$
A game is strongly monotone if $-H(\cdot)$ is a strongly monotone operator. In this case, the solution to $H(\boldx)=0$ is precisely the unique pure NE of the game. Strongly monotone games are well-studied and include, for example, a large class of resource allocation games and strongly concave potential games~\cite{VI}. 

We study learning of the NE in a stochastic distributed setting. At each iteration $k$, the player $n$ observes a noisy gradient of their utility, i.e., $\nabla_{x^{(n)}}u_n(\boldx_k)+\eta^{(n)}_k$, and updates its action  via gradient ascent using this noisy sample
\begin{align}\label{iter-game}
    x_{k+1}^{(n)}=x_k^{(n)}+\beta_k(\nabla_{x^{(n)}}u_n(\boldx_k)+\eta^{(n)}_k).
\end{align}
Aggregating the above iteration for all players, yields the joint iteration as follows,
$$\boldx_{k+1}=\boldx_k+\beta_k(H(\boldx_k)+\eta_k),$$
where $\eta_k=(\eta_k^{(1)},\ldots, \eta_k^{(N)})$ is the stacked noise vector. This scheme is of the form of iteration~\eqref{eqn:iter} with $-H(\cdot)$ being a strongly monotone operator. 

\begin{figure*}[t]
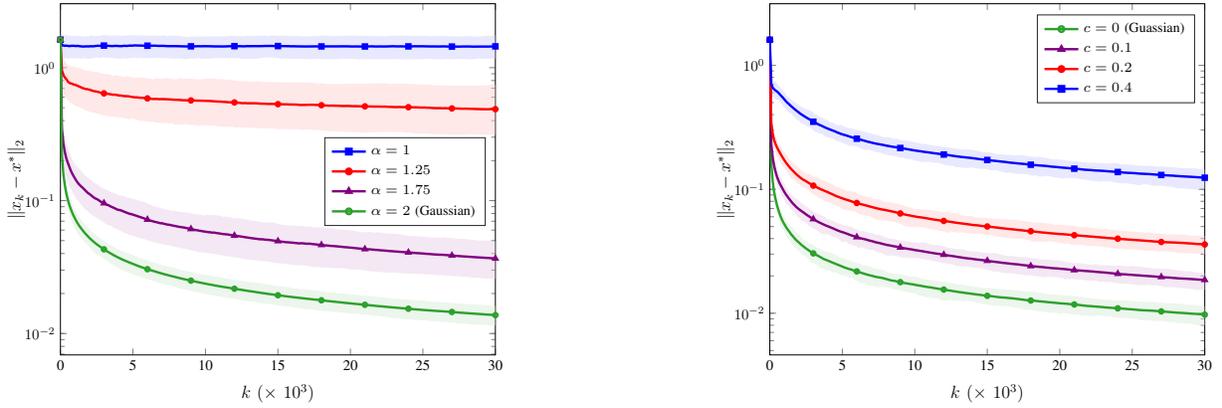

\centering
\begin{subfigure}[t]{0.48\textwidth}
  \centering
  \scalebox{0.75}{\input{figures/games_alpha_stable}}  \vspace{2mm}
  \caption{Heavy-tailed noise ($\alpha-$stable distribution)}
  \label{fig:games-ht}
\end{subfigure}\hfill
\hspace*{-3mm}
\begin{subfigure}[t]{0.48\textwidth}
  \centering
  \scalebox{0.75}{\input{figures/games_farima}}
    \vspace{2mm}
  \caption{LRD noise $(\mathrm{FARIMA}(0,c,0))$}
  \label{fig:games-lrd}
\end{subfigure}
\caption{Performance of gradient play under heavy-tailed and LRD noise models.}
\label{fig:games}
\end{figure*}

Learning of NE in games has been widely studied in the noiseless setting~\cite{VI,gamelearningtheory}. More recently, attention has shifted to stochastic settings, primarily under i.i.d.\ or bounded martingale difference noise assumptions~\cite{noisygamelearning1,noisygamelearning2}. However, such models often fail to capture realistic feedback in large multi-agent systems. Players’ utilities are often driven by shared stochastic processes that are temporally correlated rather than independent. For instance, in electricity grids and other power-system applications, aggregate demand, renewable generation, and market prices exhibit persistence and pronounced spikes~\cite{energy-example1, energy-example2}. In wireless power control, interference and traffic loads display self-similarity, long-range dependence, and bursty behavior~\cite{wireless-example1, wireless-example2}. As a result, the gradient noise arising in these settings is typically both temporally correlated and heavy‑tailed. 

As the learning scheme fits our general SA framework, our finite-time bounds (Theorems~\ref{thm:heavy} \&~\ref{thm:LRD}) apply under the corresponding noise models.\\

\noindent \textbf{Numerical Simulations.} 
We consider power control in wireless networks with $N=12$ links and $D=4$ parallel channels. Each user $n \in [N]$ chooses a power allocation vector $x^{(n)} \in \mathbb{R}^D$, where $x^{(n,d)}$ denotes the transmission power allocated to channel $d \in [D]$. The feasible strategy set of each user is
\begin{align*}
    \mathcal{X}_n=\left\{x^{(n)} \middle | \; x^{(n,d)} \geq 0\ \; \forall d\in[D], \;\; \sum_{d=1}^D x^{(n,d)} \le 1\right\},
\end{align*}
so that power allocations are non-negative and satisfy a per-user sum-power (hard) constraint. The wireless environment is described by channel gain coefficients $g_{m,n}^{(d)}$, where $g_{m,n}^{(d)}$ represents the gain from transmitter of link $m$ to receiver of link $n$ on channel $d$. The interference link  $n$ experiences on channel $d$ is given by $$I_{n}^{(d)}(\mathbf{x})=\sum_{m\neq n}g_{m,n}^{(d)}x^{(m,d)}.$$ The utility received by player $n$ is its achievable throughput or $\log_2(1+\text{SINR})$:
$$u_n(\mathbf{x})=\sum_{d=1}^D \log_2\left(1+\frac{g_{n,n}^{(d)}}{N_0+I_n^{(d)}(\mathbf{x})}\right),$$
where $N_0$ is the variance of the Gaussian noise in the channel. 

To ensure that the game is strongly monotone on the feasible set, we operate in the low-interference regime: the direct gains are sampled independently and uniformly from the interval $[0.8,1.2]$, whereas the interference gains are sampled independently and uniformly from the interval $[0.01,0.05]$. This choice ensures that direct transmission effects dominate cross-user interference. Further, we choose $N_0=1$ for our simulations. We compute the NE $\mathbf{x}^*$ numerically by performing a projected fixed-step iteration. 

Due to the constraints on the action space, the players use a projected gradient ascent iteration, projecting their actions to the set $\mathcal{X}_n$ after each iteration. The iterates are initialized randomly in the feasible set, and the stepsize sequence is chosen as $\beta_k=1/(k+1)$.

For the heavy-tailed experiments, we consider i.i.d.\ coordinate-wise symmetric $\alpha$-stable noise with stability index $\alpha$, skewness parameter $\beta=0$, scale $\sigma=0.2$, and location parameter $\mu=0$. The stability index $\alpha$ determines the tail behavior of the noise: smaller values of $\alpha$ correspond to heavier tails and more frequent large jumps, while $\alpha=2$ reduces to the Gaussian distribution $\mathcal{N}(0,2\sigma^2)$. Compared with Gaussian noise, $\alpha$-stable noise exhibits occasional large outliers, which are more realistic in interference-prone wireless environments. The noise has finite $p^{\text{th}}$ moments only for $p \le \alpha$. In Figure~\ref{fig:games-ht}, we compare the convergence of gradient play for different values of $\alpha$. As expected, smaller values of $\alpha$ lead to slower convergence. For $\alpha=1$, a regime that is not covered by our theoretical analysis, the iterates do not appear to converge.

For the LRD experiments, we consider coordinate-wise independent Gaussian FARIMA$(0,c,0)$ noise, where the memory parameter satisfies $c \in [0,1/2)$. In our simulations, we consider several values of $c$ and generate the noise using the truncated moving-average representation
\begin{align}
\eta_t = \sigma \sum_{j=0}^{L} \psi_j \epsilon_{t-j},
\qquad \epsilon_t \sim \mathcal{N}(0,1),
\end{align}
where $\sigma=0.2$ is the noise scale, $L=500$ is the truncation level, and $\{\psi_j\}$ are the moving-average coefficients. The resulting perturbations are zero-mean Gaussian with temporal dependence controlled by $c$: when $c=0$, the noise reduces to standard white Gaussian noise, while larger values of $c$ correspond to stronger temporal correlations. In Figure~\ref{fig:games-lrd}, we compare the convergence behavior of gradient play for different values of $c$.

\section{Conclusion \& Future Work}\label{sec:conc}
In this paper, we establish finite-time bounds for stochastic approximation under heavy-tailed and long-range dependent noise when the underlying operator is strongly monotone. Our analysis is based on noise-averaging framework with auxiliary iterates, which enable sharp moment bounds under these non-classical noise models. We further apply this framework to SGD and gradient play, demonstrating the impact of heavy tails and temporal dependence on convergence rates. Promising directions for future work include extending these results to non-expansive fixed-point iterations, beyond the contractive setting considered here as well as studying modified algorithms that incorporate techniques such as clipping and normalization, which are commonly used in SGD.

\appendices
\section{Proofs from Section~\ref{sec:outline}}\label{app:lemma_proof}
\subsection{Proof for Lemma~\ref{lemma:contrac}}
Recall that $G(x)=x-\zeta F(x)$, where $\zeta=\mu/L^2$. Then, 
\begin{align*}
    &\|G(x_1)-G(x_2)\|^2\\
    &=\|x_1-x_2-\zeta(F(x_1)-F(x_2))\|^2\\
    &= \|x_1-x_2\|^2+\zeta^2\|F(x_2)-F(x_1)\|^2\\
    &\;\;-2\zeta\langle x_1-x_2, F(x_1)-F(x_2) \rangle\\
    &\stackrel{(a)}{\leq} \|x_1-x_2\|^2+\zeta^2L^2\|x_1-x_2\|^2-2\mu\zeta\|x_1-x_2\|^2\\
    &=(1+\zeta^2L^2-2\mu\zeta)\|x_1-x_2\|^2.
\end{align*}
Here, inequality (a) follows from the $\mu$-strongly monotone and $L$-Lipschitz nature of operator $F(\cdot)$. Now, $(1+\zeta^2L^2-2\mu\zeta)=1-\mu^2/L^2$. Hence,
$$\|G(x_1)-G(x_2)\|\leq \left(\sqrt{1-\frac{\mu^2}{L^2}}\right)\|x_1-x_2\|.$$

Note that strong monotone nature and Lipschitzness of $F(\cdot)$ imply that $\mu \leq L$. This implies that the map $G(\cdot)$ is $\lambda$-contractive where $\lambda=\sqrt{1-\mu^2/L^2}$. Note that $\xstar$ is the unique point such that $F(\xstar)=0$. This implies that $G(\xstar)=\xstar$, and hence $\xstar$ is the unique fixed point of the mapping $G(\cdot)$. This completes the proof for part (a) of Lemma~\ref{lemma:contrac}. For part (b), note that
\begin{align*}
    x_{k+1}&=x_k-\beta_k(F(x_k)+\eta_k)\\
    &=x_k-\betatilde_k(\zeta F(x_k)+\zeta\eta_k)\\
    &=x_k+\betatilde_k(x_k-\zeta F(x_k)-x_k-\zeta\eta_k)\\
    &=x_k+\betatilde(G(x_k)-x_k+\etatilde_k).
\end{align*}
Here $\betatilde_k=\beta_k/\zeta$ and $\etatilde_k=-\zeta\eta_k$. This completes the proof for Lemma~\ref{lemma:contrac}.

\subsection{Proof for Lemma~\ref{lemma:noise_avg}}
Note that $x_k-\xstar=z_k-\xstar+x_k-z_k=z_k-\xstar+U_k$. Using the triangle inequality, we get
$$\|x_k-\xstar\|\leq \|z_k-\xstar\|+\|U_k\|.$$
For $q \geq 1$, the function $x\mapsto x^q$ is convex and hence,
$$\left(\frac{a+b}{2}\right)^q\leq \frac{a^q+b^q}{2}.$$
This implies that 
$$(a+b)^q\leq 2^q\left(\frac{a^q+b^q}{2}\right)\leq 2^{q-1}(a^q+b^q).$$
For $q \in [1,2]$, we get $(a+b)^q\leq 2(a^q+b^q)$. Substituting $a=\|z_k-\xstar\|$ and $b=\|U_k\|$ completes the proof for part (a). For part (b), recall that
\begin{align*}
    x_{k+1}&=x_k+\betatilde(G(x_k)-x_k+\etatilde_k).
\end{align*}
By definition, $x_k=z_k+U_k$. This implies
\begin{align*}
    &z_{k+1}+U_{k+1}=z_k+U_k+\betatilde_k(G(x_k)-z_k-U_k+\etatilde_k)\\
    &\implies z_{k+1}=z_k+\betatilde_k(G(x_k)-z_k)\\
    &\implies z_{k+1}=z_k+\betatilde_k(G(z_k)-z_k+\Delta_k).
\end{align*}
Here, $\Delta_k=G(x_k)-G(z_k)$. Using contractive nature of the map $G(\cdot)$, $\|\Delta_k\|\leq \|x_k-z_k\|=\|U_k\|$. This completes the proof for Lemma~\ref{lemma:noise_avg}.

\subsection{Proof for Lemma~\ref{lemma:heavy}}
Recall that $$U_k=\sum_{i=0}^{k-1}\betatilde_i\prod_{j=i+1}^{k-1}(1-\betatilde_j)\etatilde_i.$$
We use the following von Bahr-Essen-type inequality~\cite[Theorem $3.1$]{Pinelis} which shows that for independent and zero-mean random vectors $\{Y_i\}_{i\geq 0}$ with finite $p^\text{th}$ moment, we have the following:
$$\EE\left[\left\|\sum_{i=0}^{k-1}Y_i\right\|^p\right]\leq 2\sum_{i=0}^{k-1}\EE\left[\|Y_i\|^p\right].$$
Defining $Y_i=\betatilde_i\prod_{j=i+1}^{k-1}(1-\betatilde_j)\etatilde_i$, we get
\begin{align*}
    \EE\left[\|U_k\|^p\right]&\leq 2\sum_{i=0}^{k-1}\EE\left[\left\|\betatilde_i\prod_{j=i+1}^{k-1}(1-\betatilde_j)\etatilde_i\right\|^p\right]\\
    &\leq 2\sum_{i=0}^{k-1}\betatilde_i^p\prod_{j=i+1}^{k-1}(1-\betatilde_j)^p\EE\left[\|\eta_i\|^p\right]\\
    &\leq 2\zeta^p\sigma^p\sum_{i=0}^{k-1}\betatilde_i^p\prod_{j=i+1}^{k-1}(1-\betatilde_j).
\end{align*}
Here the second inequality follows from $\betatilde_i\leq 1$, which holds because $K_0\geq C_5$. The third inequality stems from Assumption~\ref{assu:heavy} and the fact that $\etatilde_k=-\zeta\eta_k$. Next, we apply Lemma~\ref{lemma:aux2} with $\afrak=1, \epsilon=\betatilde^p, \phi=\betatilde$, and $\efrak=p$. For $\betatilde\geq 2(p-1)$, which follows from $\beta\geq C_4$, we have
$$\sum_{i=0}^{k-1}\betatilde_{i}^p\prod_{j=i+1}^{k-1}(1-\betatilde_j)\leq 2\betatilde_k^{p-1}.$$
This implies that $$\EE[\|U_k\|^p]\leq 4\zeta^p\sigma^p\betatilde_k^{p-1}= 4\zeta\sigma^p\left(\frac{\beta}{k+K_0}\right)^{p-1},$$ which completes part (a) of Lemma~\ref{lemma:heavy}.

For the second part of the lemma, we first recall the modified iteration \eqref{iter-alter}. 
$$z_{k+1}=z_k+\betatilde_k(G(z_k)-z_k+\Delta_k).$$
Using the property that $\xstar$ is a fixed point for $G(\cdot)$, we get
\begin{align*}
    z_{k+1}-\xstar&=z_k-\xstar+\betatilde_k(G(z_k)-G(\xstar)-z_k+\xstar+\Delta_k)\\
    &=(1-\betatilde_k)(z_k-\xstar)+\betatilde_k(G(z_k)-G(\xstar)+\Delta_k).
\end{align*}
Now, we use the property that the function $G(\cdot)$ is $\lambda$-contractive to get the following.
\begin{align}\label{recursion-step-1}
    \|z_{k+1}-\xstar\|\leq (1-\lambda'\betatilde_k)\|z_k-\xstar\|+\betatilde_k\|\Delta_k\|,
\end{align}
where $\lambda'=1-\lambda$. 

Now, for $b,c>0$,
$$(b+c)^p-b^p = p\int_{b}^{b+c}t^{p-1}~dt=p\int_{0}^c(b+t)^{p-1}~dt.$$
Note that $0<p-1<1$ which implies that $x^{p-1}$ is concave and consequently subadditive. This implies that $(b+t)^{p-1}\leq b^{p-1}+t^{p-1}$ for $b,t\geq 0$. So,
$$(b+c)^p-b^p\leq p\int_{0}^c\left(b^{p-1}+t^{p-1}\right)~dt=pb^{p-1}c+c^p.$$
Therefore, 
$$(b+c)^p\leq b^p+ pb^{p-1}c+c^p\leq b^p+2b^{p-1}c+c^p.$$
Here we use the fact that $p<2$. Substituting $b=(1-\lambda'\betatilde_k)\|z_k-\xstar\|$ and $c=\betatilde_k\|\Delta_k\|$, we get
\begin{align}\label{heavy-split}
    \|z_{k+1}-\xstar\|^p&\leq (1-\lambda'\betatilde_k)^p\|z_k-\xstar\|^p+\betatilde_k^p\|\Delta_k\|^p\nonumber\\
    &\;\;+2(1-\lambda'\betatilde_k)^{p-1}\|z_k-\xstar\|^{p-1}\betatilde_k\|\Delta_k\|\nonumber\\
    &\leq (1-\lambda'\betatilde_k)\|z_k-\xstar\|^p+\betatilde_k^p\|\Delta_k\|^p\nonumber\\
    &\;\;+2\betatilde_k\|z_k-\xstar\|^{p-1}\|\Delta_k\|.
\end{align}
Here the second inequality follows from the fact that $1-\lambda'\betatilde_k\leq 1$. Now, we handle the last term using the Young's inequality:
$$ab\leq \frac{a^q}{q}+\frac{b^r}{r},$$
where $a,b\geq 0$ and $q,r>1$ such that $1/q+1/r=1$.
Setting $q=p/(p-1), r=p, a=(\lambda'/4)^{\frac{p-1}{p}}\|z_{k}-\xstar\|^{p-1}$ and $b=(\lambda'/4)^{-\frac{p-1}{p}}\|\Delta_k\|$, we get
\begin{align*}
    &\|z_k-\xstar\|^{p-1}\|\Delta_k\|\\
    &\leq \frac{p-1}{p}\frac{\lambda'}{4}\|z_k-\xstar\|^p+\frac{1}{p}\left(\frac{\lambda'}{4}\right)^{-(p-1)}\|\Delta_k\|^p\\
    &\leq \frac{\lambda'}{4}\|z_k-\xstar\|^p+\frac{4}{\lambda'}\|\Delta_k\|^p
\end{align*}
The last inequality here follows from the fact $1<p<2$ and $\lambda'<1$. Hence, this gives us the following intermediate bound on the third term from \eqref{heavy-split}.
$$2\betatilde_k\|z_k-\xstar\|^{p-1}\|\Delta_k\|\leq \frac{\lambda'\betatilde_k}{2}\|z_k-\xstar\|^p+\frac{8\betatilde_k}{\lambda'}\|\Delta_k\|^p.$$
Returning to \eqref{heavy-split}, we get
\begin{align*}
    \|z_{k+1}-\xstar\|^p&\leq (1-\lambda'\betatilde_k)\|z_k-\xstar\|^p+\betatilde_k^p\|\Delta_k\|^p\\
    &\;\; \frac{\lambda'\betatilde_k}{2}\|z_k-\xstar\|^p+\frac{8\betatilde_k}{\lambda'}\|\Delta_k\|^p\\
    &\leq \left(1-\frac{\lambda'}{2}\betatilde_k\right)\|z_k-\xstar\|^p+\frac{9\betatilde_k}{\lambda'}\|\Delta_k\|^p.
\end{align*}
In the second inequality here, we use the fact that $p>1$ and that $\betatilde_k\leq 1$ for all $k$. Using the fact that $\|\Delta_k\|\leq \|U_k\|$ (Lemma~\ref{lemma:noise_avg}) and taking expectation, we get
\begin{align*}
    &\EE\left[\|z_{k+1}-\xstar\|^p\right]\\
    &\leq \left(1-\frac{\lambda'}{2}\betatilde_k\right)\EE\left[\|z_k-\xstar\|^p\right]+\frac{9\betatilde_k}{\lambda'}\EE\left[\|U_k\|^p\right]\\
    &\leq \left(1-\frac{\lambda'}{2}\betatilde_k\right)\EE\left[\|z_k-\xstar\|^p\right]+\frac{36\zeta^p\sigma^p}{\lambda'}\betatilde_k^{p}.
\end{align*}
Here the second inequality follows from the bound on $\EE[\|U_k\|^p]$. Iterating from $i=0$ to $k-1$ gives us.
\begin{align*}
    \EE\left[\|z_k-\xstar\|^p\right]&\leq \|x_0-\xstar\|^p\prod_{i=0}^{k-1}\left(1-\frac{\lambda'}{2}\betatilde_i\right)\\
    &\;\;+\frac{36\zeta^p\sigma^p}{\lambda'}\sum_{i=0}^{k-1}\betatilde_i^{p}\prod_{j=i+1}^{k-1}\left(1-\frac{\lambda'}{2}\betatilde_j\right)
\end{align*}
Here we use the fact that $z_0=x_0$. Using Lemma~\ref{lemma:aux1}, for $\betatilde>2/\lambda'$, which follows from $\beta\geq C_4$, we have
$$\prod_{i=0}^{k-1}\left(1-\frac{\lambda'}{2}\betatilde_i\right)\leq \frac{K_0}{k+K_0}.$$
Using Lemma~\ref{lemma:aux2} with $\afrak=\lambda'/2, \epsilon=\betatilde^p, \phi=\betatilde$, and $\efrak=p$. For $\betatilde\geq 4(p-1)/\lambda'$, which follows from $\beta\geq C_4$, we have
$$\sum_{i=0}^{k-1}\betatilde_i^{p}\prod_{j=i+1}^{k-1}\left(1-\frac{\lambda'}{2}\betatilde_j\right)\leq \frac{4}{\lambda'}\betatilde_k^{p-1}.$$
Then, 
\begin{align*}
    &\EE\left[\|z_k-\xstar\|^p\right]\\
    &\leq \|x_0-\xstar\|^p\left(\frac{K_0}{k+K_0}\right)+\frac{144\zeta^p\sigma^p}{(1-\lambda)^2}\betatilde_k^{p-1}\\
    &\leq \|x_0-\xstar\|^p\left(\frac{K_0}{k+K_0}\right)+\frac{144\zeta\sigma^p}{(1-\lambda)^2}\left(\frac{\beta}{k+K_0}\right)^{p-1}.
\end{align*}

\subsection{Proof for Lemma~\ref{lemma:LRD}}
Under Assumption~\ref{assu:LRD}, $\{\eta_k\}_{k \geq 0}$ is a zero-mean and weakly stationary process which implies that $\EE[\eta_k]=0$ for all $k$, and $\EE[\langle\eta_i,\eta_j\rangle]=\EE[\langle\eta_0,\eta_{j-i}\rangle]=\gamma(j-i)$ for all $j\geq i$, respectively. In particular, $|\gamma(h)|\leq \sigma^2(1+h)^{-\delta}$ for $\delta\in(0,1)$. Let $$w_{i,k}=\betatilde_i\prod_{j=i+1}^{k-1}(1-\betatilde_j).$$ Then $U_k=\sum_{i=0}^{k-1}w_{i,k}\etatilde_i$ where $\etatilde_k=-\zeta\eta_k$. We define $\gammatilde(h)=\EE[\langle\eta_0,\eta_{h}\rangle]\leq \zeta^2\sigma^2(1+h)^{-\delta}$.

Therefore,
\begin{align*}
    \EE\left[\|U_k\|^2\right]
    &= \EE\left[\left\langle \sum_{i=0}^{k-1} w_{i,k}\etatilde_i,\ \sum_{j=0}^{k-1} w_{j,k}\etatilde_j \right\rangle\right] \notag \\
    &= \sum_{i=0}^{k-1} \sum_{j=0}^{k-1} w_{i,k} w_{j,k} \EE\left[\langle \etatilde_i,\etatilde_j\rangle\right] \notag \\
    &= \sum_{i=0}^{k-1} \sum_{j=0}^{k-1} w_{i,k} w_{j,k} \gammatilde(|i-j|)\\
    &= \sum_{i=0}^{k-1} w_{i,k}^2 \gammatilde(0)
    + 2 \sum_{i=0}^{k-1} \sum_{j=i+1}^{k-1} w_{i,k} w_{j,k} \gammatilde(j-i).
\end{align*}
Reparameterizing the inner summation in terms of the lag $h = j-i$, and interchanging the summations, we get 
\begin{align*}
    \sum_{i=0}^{k-1} \sum_{j=i+1}^{k-1} w_{i,k} w_{j,k} \gammatilde(j-i)&= \sum_{i=0}^{k-1} \sum_{h=1}^{k-1-i} w_{i,k} w_{i+h,k} \gammatilde(h)\\
    &=\sum_{h=1}^{k-1} \gammatilde(h)
    \sum_{i=0}^{k-1-h} w_{i,k} w_{i+h,k}.
\end{align*}
Hence,
\begin{align}\label{LRD-noise-avg}
    &\EE\left[\|U_k\|^2\right]\notag\\
    &= \gammatilde(0) \sum_{i=0}^{k-1} w_{i,k}^2
    + 2 \sum_{h=1}^{k-1} \gammatilde(h)
    \sum_{i=0}^{k-1-h} w_{i,k} w_{i+h,k}\notag\\
    &\leq |\gammatilde(0)| \sum_{i=0}^{k-1} w_{i,k}^2
    + 2 \sum_{h=1}^{k-1} |\gammatilde(h)|
    \left|
        \sum_{i=0}^{k-1-h} w_{i,k} w_{i+h,k}
    \right|.
\end{align}
Next, we bound $|\sum_{i=0}^{k-1-h} w_{i,k} w_{i+h,k}|$ using Cauchy-Schwarz inequality:
\begin{align*}
    \left|\sum_{i=0}^{k-1-h} w_{i,k} w_{i+h,k}\right|&\leq \sqrt{\sum_{i=0}^{k-1-h} w_{i,k}^2}\sqrt{\sum_{i=0}^{k-1-h} w_{i+h,k}^2}\\
    &\leq \sqrt{\sum_{i=0}^{k-1} w_{i,k}^2}\sqrt{\sum_{i=0}^{k-1} w_{i,k}^2}\\
    &\leq \sum_{i=0}^{k-1} w_{i,k}^2.
\end{align*}
Next, we use Lemma~\ref{lemma:aux2} with $\afrak=1, \epsilon=\betatilde^2, \phi=\betatilde$, and $\efrak=2$. For $\betatilde\geq 2$, which holds because $\beta\geq C_7$, we have
\begin{align*}
    \sum_{i=0}^{k-1} w_{i,k}^2&\leq \sum_{i=0}^{k-1} \betatilde_i^2\prod_{j=i+1}^{k-1}(1-\betatilde_j)^2\\
    &\leq \sum_{i=0}^{k-1} \betatilde_i^2\prod_{j=i+1}^{k-1}(1-\betatilde_j)\\
    &\leq 2\betatilde_k.
\end{align*}
For the second inequality here, we use the assumption that $\betatilde_k\leq 1$ for all $k$. Returning to \eqref{LRD-noise-avg}, we have
\begin{align*}
    \EE\left[\|U_k\|^2\right]\leq 2|\gammatilde(0)| \betatilde_k
    + 4 \sum_{h=1}^{k-1} |\gammatilde(h)|\betatilde_k.
\end{align*}
For the first term, note that $|\gammatilde(0)|\leq \zeta^2\sigma^2$. To bound the second term, note that
\begin{align*}
    \sum_{h=1}^{k-1} |\gammatilde(h)|&\leq \zeta^2\sigma^2\sum_{h=1}^{k-1}(1+h)^{-\delta}\\
    &\leq \zeta^2\sigma^2\left(1+\int_{1}^{k} x^{-\delta}dx\right)\leq \frac{\zeta^2\sigma^2k^{1-\delta}}{1-\delta}.
\end{align*}
Then, returning to the bound on $\EE[\|U_k\|^2]$, we get 
$$\EE\left[\|U_k\|^2\right]\leq \left(2+\frac{4k^{1-\delta}}{1-\delta}\right)\zeta^2\sigma^2\betatilde_k\leq \frac{6\zeta^2\sigma^2}{1-\delta}k^{1-\delta}\betatilde_k.$$
This completes the proof for part (a) of Lemma~\ref{lemma:LRD}.

For the bound on $\EE[\|z_k-\xstar\|^2]$, we first repeat the steps till \eqref{recursion-step-1} from the proof of Lemma~\ref{lemma:heavy} to get
$$\|z_{k+1}-\xstar\|\leq (1-\lambda'\betatilde_k)\|z_k-\xstar\|+\betatilde\|\Delta_k\|.$$
Squaring both sides, we get
\begin{align*}
    \|z_{k+1}-\xstar\|^2&\leq (1-\lambda'\betatilde_k)^2\|z_k-\xstar\|^2+\betatilde_k^2\|\Delta_k\|^2\\
    &\;\;+2(1-\lambda'\betatilde_k)\betatilde_k\|z_k-\xstar\|\|\Delta_k\|\\
    &\leq (1-\lambda'\betatilde_k)\|z_k-\xstar\|^2+\betatilde_k^2\|\Delta_k\|^2\\
    &\;\;+2\betatilde_k\|z_k-\xstar\|\|\Delta_k\|.
\end{align*}
Here the second inequality follows from the fact that $(1-\lambda'\betatilde_k)\leq 1$ for all $k\geq 0$. For the last term here we apply the weighted AM-GM inequality ($2ab\leq \eta a^2+(1/\eta)b^2$) with $\eta=\lambda'/2, a=\|z_k-\xstar\|$, and $b=\|\Delta_k\|$ to get 
$$2\betatilde_k\|z_k-\xstar\|\|\Delta_k\|\leq \frac{\lambda'}{2}\betatilde_k\|z_k-\xstar\|^2+\frac{2}{\lambda'}\betatilde_k\|\Delta_k\|^2.$$
This implies
\begin{align*}
    &\|z_{k+1}-\xstar\|^2\\
    &\leq \left(1-\frac{\lambda'}{2}\betatilde_k\right)\|z_k-\xstar\|^2+\left(\betatilde_k^2+\frac{2}{\lambda'}\betatilde_k\right)\|\Delta_k\|^2\\
    &\leq \left(1-\frac{\lambda'}{2}\betatilde_k\right)\|z_k-\xstar\|^2+\frac{3}{\lambda'}\betatilde_k\|U_k\|^2.
\end{align*}
Here the second inequality follows from the fact that $\betatilde_k\leq 1$ and Lemma~\ref{lemma:noise_avg} which states that $\|\Delta_k\|\leq \|U_k\|$. Taking expectation, we get 
\begin{align*}
    &\EE\left[\|z_{k+1}-\xstar\|^2\right]\\
    &\leq \left(1-\frac{\lambda'}{2}\betatilde_k\right)\EE\left[\|z_k-\xstar\|^2\right]+\frac{3}{\lambda'}\betatilde_k\EE\left[\|U_k\|^2\right]\\
    &\leq \left(1-\frac{\lambda'}{2}\betatilde_k\right)\EE\left[\|z_k-\xstar\|^2\right]+\frac{18\zeta^2\sigma^2}{\lambda'(1-\delta)}k^{1-\delta}\betatilde_k^2.
\end{align*}
Iterating from $i=0$ to $k-1$, we get 
\begin{align*}
    \EE\left[\|z_k-\xstar\|^2\right]&\leq \|x_0-\xstar\|^2\prod_{i=0}^{k-1} \left(1-\frac{\lambda'}{2}\betatilde_i\right)\\
    &\;\;+\frac{18\zeta^2\sigma^2}{\lambda'(1-\delta)}\sum_{i=0}^{k-1}i^{1-\delta} \betatilde_i^2\prod_{j=i+1}^{k-1}\left(1-\frac{\lambda'}{2}\betatilde_j\right)
\end{align*}
Using Lemma~\ref{lemma:aux1}, for $\betatilde>2/\lambda'$, which follows from $\beta\geq C_7$, we have
$$\prod_{i=0}^{k-1}\left(1-\frac{\lambda'}{2}\betatilde_i\right)\leq \frac{K_0}{k+K_0}.$$
For the second term, we first note that $$i^{1-\delta} \betatilde_i^2\leq (i+K_0)^{1-\delta}\betatilde_i^2=\frac{\betatilde^2}{(i+K_0)^{1+\delta}}.$$
Applying Lemma~\ref{lemma:aux2} with $\afrak=\lambda'/2, \epsilon=\betatilde^2, \phi=\betatilde$, and $\efrak=1+\delta$. For $\betatilde\geq 4\delta/\lambda'$, which follows from $\beta\geq C_7$, we get 
\begin{align*}
    \sum_{i=0}^{k-1}\frac{\betatilde^2}{(i+K_0)^{1+\delta}}\prod_{j=i+1}^{k-1}\left(1-\frac{\lambda'}{2}\betatilde_j\right)\leq \frac{4}{\lambda'}\frac{\betatilde}{(k+K_0)^\delta}.
\end{align*}
Hence,
\begin{align*}
    &\EE\left[\|z_k-\xstar\|^2\right]\\
    &\leq \|x_0-\xstar\|^2\frac{K_0}{k+K_0}+\frac{72\zeta^2\sigma^2}{(1-\lambda)^2(1-\delta)}\frac{\betatilde}{(k+K_0)^\delta}\\
    &= \|x_0-\xstar\|^2\frac{K_0}{k+K_0}+\frac{72\zeta\sigma^2}{(1-\lambda)^2(1-\delta)}\frac{\beta}{(k+K_0)^\delta},
\end{align*}
where the last equality follows from the definition $\betatilde=\beta/\zeta$. This completes the proof for Lemma~\ref{lemma:LRD}.

\section{Proofs for Theorems~\ref{thm:standard},~\ref{thm:heavy} and~\ref{thm:LRD}}\label{app:thm_proof}
\subsection{Proof for Theorem~\ref{thm:standard}}\label{app:thm_proof_standard}
Subtracting $\xstar$ on both sides from \eqref{eqn:iter} gives us
$$x_{k+1}-\xstar=x_k-\xstar-\beta_k(F(x_k)+\eta_k).$$
This implies
\begin{align}\label{standard-eqn-1}
    \|x_{k+1}-\xstar\|^2&=\|x_k-\xstar-\beta_kF(x_k)\|^2\notag\\
    &\;\;+2\langle x_k-\xstar-\beta_kF(x_k),\beta_k\eta_k\rangle+\beta_k^2\|\eta_k\|^2.
\end{align}
For the first term, note that 
\begin{align*}
    &\|x_k-\xstar-\beta_kF(x_k)\|^2\\
    &=\|x_k-\xstar\|^2+\beta_k^2\|F(x_k)\|^2-2\beta_k\langle x_k-\xstar,F(x_k)\rangle\\
    &=\|x_k-\xstar\|^2+\beta_k^2\|F(x_k)-F(\xstar)\|^2\\
    &\;\;-2\beta_k\langle x_k-\xstar,F(x_k)-F(\xstar)\rangle\\
    &\leq \|x_k-\xstar\|^2+L^2\beta_k^2\|x_k-\xstar\|^2-2\mu\beta_k\|x_k-\xstar\|^2.
\end{align*}
Here the second equality follows from the definition that $F(\xstar)=0$ and the inequality follows from Assumption~\ref{assu:monotone}. Under the assumption that $L^2\beta_k^2\leq \mu\beta_k$ and $\mu\beta_k\leq 1$, we get 
\begin{align*}
    \|x_k-\xstar-\beta_kF(x_k)\|^2\leq (1-\mu\beta_k)\|x_k-\xstar\|^2.
\end{align*}
Returning to \eqref{standard-eqn-1}, and taking condition expectation with respect to $\FF_k$, we get 
\begin{subequations}\label{split-standard}
    \begin{align}
    &\EE\left[\|x_{k+1}-\xstar\|^2\mid\FF_k\right]\nonumber\\
    &\leq (1-\mu\beta_k)\EE\left[\|x_k-\xstar\|^2\mid\FF_k\right]\label{split-standard-1}\\
    &\;\;+2\EE\left[\langle x_k-\xstar-\beta_kF(x_k),\beta_k\eta_k\rangle\mid\FF_k\right]\label{split-standard-2}\\
    &\;\;+\beta_k^2\EE\left[\|\eta_k\|^2\mid\FF_k\right]\label{split-standard-3}.
\end{align}
\end{subequations}
Then under the assumption that $\eta_k$ is a light-tailed martingale difference sequence, the term \eqref{split-standard-2} is zero and the term \eqref{split-standard-3} can be bounded by $\beta_k^2\sigma^2$. Taking expectation, \eqref{split-standard} can then be simplified to obtain the following recursion.
\begin{align*}
    \EE\left[\|x_{k+1}-\xstar\|^2\right]\leq (1-\mu\beta_k)\EE\left[\|x_k-\xstar\|^2\right]+\beta_k^2\sigma^2.
\end{align*}
Iterating from $i=0$ to $k-1$, we get 
\begin{align*}
    \EE\left[\|x_k-\xstar\|^2\right]&\leq \|x_0-\xstar\|^2\prod_{i=0}^{k-1}(1-\mu\beta_i)\\
    &\;\;+\sigma^2\sum_{i=0}^{k-1}\beta_i^2\prod_{j=i+1}^{k-1}(1-\mu\beta_j).
\end{align*}
Using Lemma~\ref{lemma:aux1}, for $\beta>1/\mu$, we have
$$\prod_{i=0}^{k-1}(1-\mu\beta_i)\leq \frac{K_0}{k+K_0}.$$
Using Lemma~\ref{lemma:aux2} with $\afrak=\mu, \epsilon=\betatilde^2, \phi=\beta$, and $\efrak=2$. For $\beta\geq 2/\mu$, we have
$$\sum_{i=0}^{k-1}\beta_i^2\prod_{j=i+1}^{k-1}(1-\mu\beta_j)\leq \frac{2}{\mu}\beta_k.$$
Hence,
\begin{align*}
    \EE\left[\|x_k-\xstar\|^2\right]&\leq \|x_0-\xstar\|^2\frac{K_0}{k+K_0}+\frac{2\sigma^2}{\mu}\beta_k\\
    &=\frac{C_3}{k+K_0},
\end{align*}
where $C_3=K_0\|x_0-\xstar\|^2+2\beta\sigma^2/\mu$. This completes the proof for Theorem~\ref{thm:standard}.

\subsubsection{Values of Constants in Theorem~\ref{thm:standard}}
We assume $\beta\geq C_1$, where  $C_1=2/\mu,$ and $K_0\geq C_2$, where $C_2=\beta L^2/\mu+\beta\mu.$ The constant $C_3$ in the bound is  $C_3=K_0\|x_0-\xstar\|^2+2\beta\sigma^2/\mu$.

\subsection{Proof for Theorem~\ref{thm:heavy}}\label{app:thm_proof_heavy}
Using Lemma~\ref{lemma:noise_avg} with $q=p$, we have the following bound.
$$\EE\left[\|x_k-\xstar\|^p\right]\leq 2\EE\left[\|z_k-\xstar\|^p\right]+2\EE\left[\|U_k\|^p\right].$$
We bound the above terms using Lemma~\ref{lemma:heavy}.
\begin{align*}
    &\EE\left[\|x_k-\xstar\|^p\right]\\
    &\leq \frac{2K_0}{k+K_0}\|x_0-\xstar\|^p+\frac{288\zeta\sigma^p}{(1-\lambda)^2}\left(\frac{\beta}{k+K_0}\right)^{p-1}\\
    &\;\;+8\zeta\sigma^p\left(\frac{\beta}{k+K_0}\right)^{p-1}\\
    &\leq \frac{2K_0}{k+K_0}\|x_0-\xstar\|^p+\frac{296\zeta\sigma^p}{(1-\lambda)^2}\left(\frac{\beta}{k+K_0}\right)^{p-1}\\
    &\leq \frac{C_6}{(k+K_0)^{p-1}},
\end{align*}
where $C_6=2K_0\|x_0-\xstar\|^p+296\zeta\sigma^p\beta^{p-1}/(1-\lambda)^2$. This completes the proof for Theorem~\ref{thm:heavy}.

\subsubsection{Values of Constants in Theorem~\ref{thm:heavy}}
We assume $\beta\geq C_4$, where 
$$C_4=\frac{2+4(p-1)}{1-\sqrt{1-\mu^2/L^2}}\frac{\mu}{L^2},$$
and $K_0\geq C_5$, where
$$C_5=\beta\frac{L^2}{\mu}.$$
The constant $C_6$ in the bound is 
$$C_6=2K_0\|x_0-\xstar\|^p+296\frac{\mu}{L^2}\frac{\sigma^p\beta^{p-1}}{(1-\sqrt{1-\mu^2/L^2})^2}.$$

\subsection{Proof for Theorem~\ref{thm:LRD}}\label{app:thm_proof_LRD}
Using Lemma~\ref{lemma:noise_avg} with $q=2$, we have
$$\EE\left[\|x_k-\xstar\|^2\right]\leq 2\EE\left[\|z_k-\xstar\|^2\right]+2\EE\left[\|U_k\|^2\right].$$
We bound the above terms using Lemma~\ref{lemma:LRD}.
\begin{align*}
    &\EE\left[\|x_k-\xstar\|^2\right]\\
    &\leq \frac{2K_0}{k+K_0}\|x_0-\xstar\|^2+\frac{144\zeta\sigma^2}{(1-\lambda)^2(1-\delta)}\frac{\beta}{(k+K_0)^\delta}\\
    &\;\;+\frac{12\zeta\sigma^2}{1-\delta}\frac{\beta}{(k+K_0)^{\delta}}\\
    &\leq \frac{2K_0}{k+K_0}\|x_0-\xstar\|^2+\frac{156\zeta\sigma^2}{(1-\lambda)^2(1-\delta)}\frac{\beta}{(k+K_0)^\delta}\\
    &\leq \frac{C_9}{(k+K_0)^\delta},
\end{align*}
where $C_9=2K_0\|x_0-\xstar\|^2+156\zeta\sigma^2\beta/((1-\delta)(1-\lambda)^2)$. This completes the proof for Theorem~\ref{thm:LRD}.

\subsubsection{Values of Constants in Theorem~\ref{thm:LRD}}
We assume $\beta\geq C_7$, where 
$$C_7=\frac{2+4\delta}{1-\sqrt{1-\mu^2/L^2}}\frac{\mu}{L^2},$$
and $K_0\geq C_8$, where
$$C_8=\beta\frac{L^2}{\mu}.$$
The constant $C_9$ in the bound is 
$$C_9=2K_0\|x_0-\xstar\|^2+156\frac{\mu}{L^2}\frac{\sigma^2\beta}{(1-\delta)(1-\sqrt{1-\mu^2/L^2})^2}.$$

\section{Auxiliary Lemmas}
We present two lemmas which help us simplify the recursions typically obtained in finite-time analysis of SA, and are useful throughout this work.
\begin{lemma}\label{lemma:aux1}
Suppose $\phi_k=\phi/(k+K_0)$ for $\phi,K>0$. If $\phi>\frac{1}{\afrak}$ and $\afrak\phi_k\leq 1$, then
$$\prod_{i=0}^{k-1}(1-\afrak\phi_i)\leq \frac{K_0}{k+K_0}.$$
\end{lemma}
\begin{proof}
    Using the fact that $1+x\leq e^x$ for all $x\in\RR$,
    \begin{align*}
        \prod_{i=0}^{k-1}(1-\afrak\phi_i)&=\prod_{i=0}^{k-1}\left(1-\frac{\afrak\phi}{i+K_0}\right)\\
        &\leq \exp\left(-\afrak\phi\sum_{i=0}^{k-1}\frac{1}{i+K_0}\right)\\
        &\leq \exp\left(-\sum_{i=0}^{k-1}\frac{1}{i+K_0}\right).
    \end{align*}
    Here the final inequality follows from our assumption that $\phi\afrak>1$. Now, for any non-increasing function $h(x)$, we have that $\sum_{i=a}^{b}\geq \int_{a}^{b+1} h(x)dx.$ This implies that 
    \begin{align*}
        \sum_{i=0}^{k-1}\frac{1}{i+K_0}\geq \int_{0}^{k} \frac{1}{x+K_0}dx=\log\left(\frac{k+K_0}{K_0}\right).
    \end{align*}
    Finally, this implies that 
    $$\prod_{i=0}^{k-1}(1-\afrak\phi_i)\leq \frac{K_0}{k+K_0}.$$
This completes our proof.
\end{proof}

\begin{lemma}\label{lemma:aux2}
    Let $\phi,K, \epsilon>0$. Suppose $\phi_k=\phi/(k+K)$. Let $\epsilon_k=\epsilon/(k+K)^{\efrak}$, where $\efrak\in (1,2]$. If $\afrak> 0$, $\phi\geq \frac{2(\efrak-1)}{\afrak}$ and $\afrak\phi_k\leq 1$, then 
    $$\sum_{i=0}^{k-1}\epsilon_i\prod_{j=i+1}^{k-1}(1-\phi_j\afrak)\leq \frac{2}{\afrak}\frac{\epsilon_k}{\phi_k}.$$
\end{lemma}

\begin{proof}
    Define sequence $s_{0}=0$ and $s_{k+1}=(1-\phi_k\afrak)s_k+\epsilon_k$. Note that $s_k=\sum_{i=0}^{k-1}\epsilon_i\prod_{j=i+1}^{k-1}(1-\phi_j\afrak)$. We will use induction to show our required result. Suppose that $s_k\leq (2/\afrak)(\epsilon_k/\phi_k)$ holds for some $k$. Then,
    \begin{align*}
        \frac{2}{\afrak}\frac{\epsilon_{k+1}}{\phi_{k+1}}-s_{k+1}&=\frac{2}{\afrak}\frac{\epsilon_{k+1}}{\phi_{k+1}}-(1-\afrak\phi_k)s_k-\epsilon_k\\
        &\geq \frac{2}{\afrak}\frac{\epsilon_{k+1}}{\phi_{k+1}}-(1-\afrak\phi_k)\frac{2}{\afrak}\frac{\epsilon_k}{\phi_k}-\epsilon_k\\
        &=\frac{2}{\afrak}\left(\frac{\epsilon_{k+1}}{\phi_{k+1}}-\frac{\epsilon_k}{\phi_k}\right)+\epsilon_k.
    \end{align*}
    Here the inequality follows from our assumption that the required inequality holds at time $k$.
    Now,
    $$\left(\frac{\epsilon_{k+1}}{\phi_{k+1}}-\frac{\epsilon_k}{\phi_k}\right)=\frac{\epsilon}{\phi}\left(\frac{1}{(k+K+1)^{\efrak-1}}-\frac{1}{(k+K)^{\efrak-1}}\right).$$
    For $\efrak-1\in(0,1]$,
    \begin{align*}
        &\frac{1}{(k+K+1)^{\efrak-1}}-\frac{1}{(k+K)^{\efrak-1}}\\
        &=\frac{1}{(k+K)^{\efrak-1}}\left(\left[\left(1+\frac{1}{k+K}\right)^{k+K}\right]^{-\frac{\efrak-1}{k+K}}-1\right)\\
        &\geq \frac{1}{(k+K)^{\efrak-1}}\left(e^{-\frac{\efrak-1}{k+K}}-1\right)\\
        &\geq -\frac{1}{(k+K)^{\efrak-1}}\frac{\efrak-1}{k+K}=-\frac{\efrak-1}{\epsilon}\epsilon_k.
    \end{align*}
    Here, the first inequality follows from the inequality $(1+1/x)^x\leq e$ and $e^x\geq 1+x$ for all $x$. This implies 
    \begin{align*}
        \frac{2}{\afrak}\frac{\epsilon_{k+1}}{\phi_{k+1}}-s_{k+1}&\geq -\frac{2}{\afrak}\frac{\epsilon}{\phi}\frac{\efrak-1}{\epsilon}\epsilon_k+\epsilon_k\\
        &=\epsilon_k\left(1-\frac{2(\efrak-1)}{\afrak\phi}\right).
    \end{align*}
    Since we have the assumption that $\phi\geq \frac{2(\efrak-1)}{\afrak}$, therefore, the following holds $s_{k+1}\leq \frac{2}{\afrak}\frac{\epsilon_{k+1}}{\phi_{k+1}}$. This completes the proof by induction.
\end{proof}

\section*{References}
\bibliographystyle{IEEEtran}
\bibliography{ref}

@book{Foss-Pareto,
  title={An introduction to heavy-tailed and subexponential distributions},
  author={Foss, Sergey and Korshunov, Dmitry and Zachary, Stan and others},
  volume={6},
  year={2011},
  publisher={Springer}
}

@book{heavy-tail-book,
  title={The fundamentals of heavy tails: Properties, emergence, and estimation},
  author={Nair, Jayakrishnan and Wierman, Adam and Zwart, Bert},
  volume={53},
  year={2022},
  publisher={Cambridge University Press}
}

@book{Sutton,
  title={Reinforcement learning: An introduction},
  author={Sutton, Richard S and Barto, Andrew G},
  year={2018},
publisher={The MIT Press}
}

@article{Robbins-Monro,
author = {Herbert Robbins and Sutton Monro},
title = {{A Stochastic Approximation Method}},
volume = {22},
journal = {The Annals of Mathematical Statistics},
number = {3},
publisher = {Institute of Mathematical Statistics},
pages = {400 -- 407},
year = {1951}
}

@article{bubeck2015convex,
  title={Convex optimization: Algorithms and complexity},
  author={Bubeck, S{\'e}bastien},
  journal={Foundations and trends in Machine Learning},
  volume={8},
  number={3-4},
  pages={231--357},
  year={2015},
  publisher={Emerald Publishing limited}
}

@book{Kushner,
  title={Stochastic Approximation and Recursive Algorithms and Applications},
  author={Kushner, H. and Yin, G.G.},
  isbn={9781489926968},
  lccn={96048847},
  series={Stochastic Modelling and Applied Probability},
  year={2013},
  publisher={Springer New York}
}

@article{Borkar-Markov,
author = {Borkar, V. S. and Meyn, S. P.},
title = {The O.D.E. Method for Convergence of Stochastic Approximation and Reinforcement Learning},
journal = {SIAM Journal on Control and Optimization},
volume = {38},
number = {2},
pages = {447-469},
year = {2000}
}

@book{Borkar-book,
  title={Stochastic Approximation: A Dynamical Systems Viewpoint: Second Edition},
  author={Borkar, V.S.},
  isbn={9788195196111},
  series={Texts and Readings in Mathematics},
  year={2022},
  publisher={Hindustan Book Agency}
}

@book{Samo-alpha-stable,
  title={Stable non-Gaussian random processes: stochastic models with infinite variance},
  author={Samorodnitsky, Gennady and Taqqu, Murad S},
  volume={1},
  year={1994},
  publisher={CRC press}
}

@article{Pinelis,
  title={Multidimensional probability inequalities via spherical symmetry},
  author={Pinelis, Iosif},
  journal={arXiv preprint arXiv:2210.04391},
  year={2022}
}

@misc{ht-sgd-2,
      title={Privacy of SGD under Gaussian or Heavy-Tailed Noise: Guarantees without Gradient Clipping}, 
      author={Umut Şimşekli and Mert Gürbüzbalaban and Sinan Yıldırım and Lingjiong Zhu},
      year={2025},
      eprint={2403.02051},
      archivePrefix={arXiv},
      primaryClass={stat.ML},
      url={https://arxiv.org/abs/2403.02051}, 
}

@book{LRD-book, place={Cambridge}, series={Cambridge Series in Statistical and Probabilistic Mathematics}, title={Long-Range Dependence and Self-Similarity}, publisher={Cambridge University Press}, author={Pipiras, Vladas and Taqqu, Murad S.}, year={2017}, collection={Cambridge Series in Statistical and Probabilistic Mathematics}}

@ARTICLE{lrd-ex-1,
  author={Leland, W.E. and Taqqu, M.S. and Willinger, W. and Wilson, D.V.},
  journal={IEEE/ACM Transactions on Networking}, 
  title={On the self-similar nature of Ethernet traffic (extended version)}, 
  year={1994},
  volume={2},
  number={1},
  pages={1-15}}

@article{lrd-ex-2,
author = {Rypdal, K. and Østvand, L. and Rypdal, M.},
title = {Long-range memory in Earth's surface temperature on time scales from months to centuries},
journal = {Journal of Geophysical Research: Atmospheres},
volume = {118},
number = {13},
pages = {7046-7062},
year = {2013}
}

@article{lrd-ex-3,
	author = {Daniel O. Cajueiro and Benjamin M. Tabak},
	doi = {https://doi.org/10.1016/j.chaos.2006.09.090},
	issn = {0960-0779},
	journal = {Chaos, Solitons \& Fractals},
	number = {3},
	pages = {918-927},
	title = {Testing for long-range dependence in world stock markets},
	volume = {37},
	year = {2008}}

@article{finmarkets,
	author = {R. Cont},
	doi = {10.1080/713665670},
	eprint = {https://doi.org/10.1080/713665670},
	journal = {Quantitative Finance},
	number = {2},
	pages = {223--236},
	publisher = {Routledge},
	title = {Empirical properties of asset returns: stylized facts and statistical issues},
	volume = {1},
	year = {2001}}

@article{qq,
	author = {Whitt, Ward},
	date = {2000/11/01},
	doi = {10.1023/A:1019143505968},
	id = {Whitt2000},
	isbn = {1572-9443},
	journal = {Queueing Systems},
	number = {1},
	pages = {71--87},
	title = {The impact of a heavy-tailed service-time distribution upon the M/GI/s waiting-time distribution},
	volume = {36},
	year = {2000}}

@InProceedings{ht-sgd-1,
  title = 	 {High-probability Convergence Bounds for Online Nonlinear Stochastic Gradient Descent under Heavy-tailed Noise},
  author =       {Armacki, Aleksandar and Yu, Shuhua and Sharma, Pranay and Joshi, Gauri and Bajovic, Dragana and Jakovetic, Dusan and Kar, Soummya},
  booktitle = 	 {Proceedings of The 28th International Conference on Artificial Intelligence and Statistics},
  pages = 	 {1774--1782},
  year = 	 {2025},
  editor = 	 {Li, Yingzhen and Mandt, Stephan and Agrawal, Shipra and Khan, Emtiyaz},
  volume = 	 {258},
  series = 	 {Proceedings of Machine Learning Research},
  month = 	 {03--05 May},
  publisher =    {PMLR}
}

@inproceedings{
lrd-sgd-2,
title={Gradient Descent with Linearly Correlated Noise: Theory and Applications to Differential Privacy},
author={Anastasia Koloskova and Ryan McKenna and Zachary Charles and J Keith Rush and Hugh Brendan McMahan},
booktitle={Thirty-seventh Conference on Neural Information Processing Systems},
year={2023}
}

@article{stat,
 ISSN = {08834237},
 author = {Ying Kuen Cheung},
 journal = {Statistical Science},
 number = {2},
 pages = {191--201},
 publisher = {Institute of Mathematical Statistics},
 title = {Stochastic Approximation and Modern Model-Based Designs for Dose-Finding Clinical Trials},
 urldate = {2026-03-14},
 volume = {25},
 year = {2010}
}

@book{VI,
  title={Finite-dimensional variational inequalities and complementarity problems},
  author={Facchinei, Francisco and Pang, Jong-Shi},
  year={2007},
  publisher={Springer Science \& Business Media}
}

@article{noisygamelearning1,
  title={Doubly optimal no-regret online learning in strongly monotone games with bandit feedback},
  author={Ba, Wenjia and Lin, Tianyi and Zhang, Jiawei and Zhou, Zhengyuan},
  journal={Operations Research},
  volume={73},
  number={6},
  pages={3219--3244},
  year={2025},
  publisher={INFORMS}
}

@incollection{huber1992robust,
  title={Robust estimation of a location parameter},
  author={Huber, Peter J},
  booktitle={Breakthroughs in statistics: Methodology and distribution},
  pages={492--518},
  year={1992},
  publisher={Springer}
}

@article{noisygamelearning2,
  title={No-regret learning in games with noisy feedback: Faster rates and adaptivity via learning rate separation},
  author={Hsieh, Yu-Guan and Antonakopoulos, Kimon and Cevher, Volkan and Mertikopoulos, Panayotis},
  journal={Advances in Neural Information Processing Systems},
  volume={35},
  pages={6544--6556},
  year={2022}
}

@book{gamelearningtheory,
  title={The theory of learning in games},
  author={Fudenberg, Drew and Levine, David K},
  volume={2},
  year={1998},
  publisher={MIT press}
}

@article{Zaiwei,
  title={Finite-sample analysis of nonlinear stochastic approximation with applications in reinforcement learning},
  author={Chen, Zaiwei and Zhang, Sheng and Doan, Thinh T and Clarke, John-Paul and Maguluri, Siva Theja},
  journal={Automatica},
  volume={146},
  pages={110623},
  year={2022},
  publisher={Elsevier}
}

@inproceedings{srikant2019finite,
  title={Finite-time error bounds for linear stochastic approximation andtd learning},
  author={Ying, Lei},
  booktitle={Conference on learning theory},
  pages={2803--2830},
  year={2019},
  organization={PMLR}
}

@article{anantharam2012stochastic,
  title={Stochastic approximation with long range dependent and heavy tailed noise},
  author={Anantharam, Venkat and Borkar, Vivek S},
  journal={Queueing Systems},
  volume={71},
  number={1},
  pages={221--242},
  year={2012},
  publisher={Springer}
}

@article{Bravo,
author = {Bravo, Mario and Cominetti, Roberto},
title = {Stochastic Fixed-Point Iterations for Nonexpansive Maps: Convergence and Error Bounds},
journal = {SIAM Journal on Control and Optimization},
volume = {62},
number = {1},
pages = {191-219},
year = {2024},
doi = {10.1137/22M1515550}
}

@article{Chandak-TTS-opti,
  title={$ O (1/k) $ Finite-Time Bound for Non-Linear Two-Time-Scale Stochastic Approximation},
  author={Chandak, Siddharth},
  journal={arXiv preprint arXiv:2504.19375},
  year={2025}
}

@INPROCEEDINGS{lms,
  author={Widrow, B. and Lehr, M. and Beaufays, F. and Wan, E. and Bileillo, M.},
  booktitle={IEEE International Conference on Neural Networks}, 
  title={Learning algorithms for adaptive processing and control}, 
  year={1993},
  volume={},
  number={},
  pages={1-8 vol.1}}

@article{duchi-sgd,
  author  = {John Duchi and Elad Hazan and Yoram Singer},
  title   = {Adaptive Subgradient Methods for Online Learning and Stochastic Optimization},
  journal = {Journal of Machine Learning Research},
  year    = {2011},
  volume  = {12},
  number  = {61},
  pages   = {2121--2159}
}

@article{energy-example1,
  title={A stochastic simulation scheme for the long-term persistence, heavy-tailed and double periodic behavior of observational and reanalysis wind time-series},
  author={Katikas, Loukas and Dimitriadis, Panayiotis and Koutsoyiannis, Demetris and Kontos, Themistoklis and Kyriakidis, Phaedon},
  journal={Applied Energy},
  volume={295},
  pages={116873},
  year={2021},
  publisher={Elsevier}
}

@article{energy-example2,
  title={Heavy tails and electricity prices: Do time series models with non-Gaussian noise forecast better than their Gaussian counterparts?},
  author={Weron, Rafal and Misiorek, Adam},
  year={2007},
  publisher={Hugo Steinhaus Center, Wroclaw University of Technology}
}

@article{wireless-example1,
  title={Long-range dependence ten years of Internet traffic modeling},
  author={Karagiannis, Thomas and Molle, Mart and Faloutsos, Michalis},
  journal={IEEE internet computing},
  volume={8},
  number={5},
  pages={57--64},
  year={2004},
  publisher={IEEE}
}

@article{wireless-example2,
  title={Self-Similar Network Traffic: An Overview},
  author={Park, Kihong and Willinger, Walter},
  journal={Self-Similar Network Traffic and Performance Evaluation},
  pages={1--38},
  year={2000},
  publisher={Wiley Online Library}
}

@InProceedings{vanilla-sgd-1,
  title = 	 {A Tail-Index Analysis of Stochastic Gradient Noise in Deep Neural Networks},
  author =       {Simsekli, Umut and Sagun, Levent and Gurbuzbalaban, Mert},
  booktitle = 	 {Proceedings of the 36th International Conference on Machine Learning},
  pages = 	 {5827--5837},
  year = 	 {2019},
  editor = 	 {Chaudhuri, Kamalika and Salakhutdinov, Ruslan},
  volume = 	 {97},
  series = 	 {Proceedings of Machine Learning Research},
  month = 	 {09--15 Jun},
  publisher =    {PMLR}
}

@inproceedings{vanilla-sgd-2,
  title={Algorithmic stability of heavy-tailed stochastic gradient descent on least squares},
  author={Raj, Anant and Barsbey, Melih and Gurbuzbalaban, Mert and Zhu, Lingjiong and {\c{S}}im, Umut and others},
  booktitle={International Conference on Algorithmic Learning Theory},
  pages={1292--1342},
  year={2023},
  organization={PMLR}
}

@article{vanilla-sgd-3,
  title={Can SGD Handle Heavy-Tailed Noise?},
  author={Fatkhullin, Ilyas and H{\"u}bler, Florian and Lan, Guanghui},
  journal={arXiv preprint arXiv:2508.04860},
  year={2025}
}

@article{vanilla-sgd-4,
  title={Convergence rates of stochastic gradient descent under infinite noise variance},
  author={Wang, Hongjian and Gurbuzbalaban, Mert and Zhu, Lingjiong and Simsekli, Umut and Erdogdu, Murat A},
  journal={Advances in Neural Information Processing Systems},
  volume={34},
  pages={18866--18877},
  year={2021}
}

@article{vanilla-sgd-5,
  author  = {Wanrong Zhu and Zhipeng Lou and Wei Biao Wu},
  title   = {Beyond Sub-Gaussian Noises: Sharp Concentration Analysis for Stochastic Gradient Descent},
  journal = {Journal of Machine Learning Research},
  year    = {2022},
  volume  = {23},
  number  = {46},
  pages   = {1--22}
}

@article{clipping-1,
  title={Stochastic optimization with heavy-tailed noise via accelerated gradient clipping},
  author={Gorbunov, Eduard and Danilova, Marina and Gasnikov, Alexander},
  journal={Advances in Neural Information Processing Systems},
  volume={33},
  pages={15042--15053},
  year={2020}
}

@article{clipping-2,
  title={Eliminating sharp minima from SGD with truncated heavy-tailed noise},
  author={Wang, Xingyu and Oh, Sewoong and Rhee, Chang-Han},
  journal={arXiv preprint arXiv:2102.04297},
  year={2021}
}

@article{normalization,
  author  = {Tao Sun and Xinwang Liu and Kun Yuan},
  title   = {Revisiting Gradient Normalization and Clipping for Nonconvex SGD under Heavy-Tailed Noise: Necessity, Sufficiency, and Acceleration},
  journal = {Journal of Machine Learning Research},
  year    = {2025},
  volume  = {26},
  number  = {237},
  pages   = {1--42}
}

\end{document}